\def\year{2022}\relax
\documentclass[letterpaper]{article} 
\usepackage{aaai22}  
\usepackage{times}  
\usepackage{helvet}  
\usepackage{courier}  
\usepackage[hyphens]{url}  
\usepackage{graphicx} 
\usepackage{natbib}  
\usepackage{caption} 
\DeclareCaptionStyle{ruled}{labelfont=normalfont,labelsep=colon,strut=off} 
\usepackage{algorithm}
\usepackage{algorithmic}

\usepackage{newfloat}
\usepackage{listings}
\floatstyle{ruled}
\newfloat{listing}{tb}{lst}{}
\floatname{listing}{Listing}
\title{Fast Line Search for Multi-Task Learning}
\author {
    Andrey Filatov\equalcontrib,\textsuperscript{\rm 1,2}\\
    Daniil Merkulov\equalcontrib,\textsuperscript{\rm 1,2}\\
}
\affiliations {
    \textsuperscript{\rm 1} Skolkovo Institute of Science and Technology \\
    \textsuperscript{\rm 2} Moscow Institute of Physics and Technology \\
    filatov.av@phystech.edu, daniil.merkulov@skolkovotech.ru
}

\def\gL{{\mathcal{L}}}
\usepackage{amsmath}
\usepackage{amssymb} 
\usepackage{bm}
\usepackage{amsthm}

\theoremstyle{definition}
\newtheorem{definition}{Definition}
\newtheorem{remark}{Remark}
\usepackage{algorithmic}
\usepackage{subcaption}
\usepackage{booktabs}
\usepackage{pifont}
\usepackage{thm-restate}

\begin{document}

\maketitle

\begin{abstract}
Multi-task learning is a powerful method for solving several tasks jointly by learning robust representation. Optimization of the multi-task learning model is a more complex task than a single-task due to task conflict. Based on theoretical results, convergence to the optimal point is guaranteed when step size is chosen through line search. But, usually, line search for the step size is not the best choice due to the large computational time overhead. We propose a novel idea for line search algorithms in multi-task learning. The idea is to use latent representation space instead of parameter space for finding step size. We examined this idea with backtracking line search. We compare this fast backtracking algorithm with classical backtracking and gradient methods with a constant learning rate on MNIST, CIFAR-10, Cityscapes tasks. The systematic empirical study showed that the proposed method leads to more accurate and fast solution, than the traditional backtracking approach and keep competitive computational time and performance compared to the constant learning rate method.

\end{abstract}

\section{Introduction}

Multi-task Learning (MTL) \cite{caruana1997multitask, ruder2017overview, zhang2017survey, standley2019tasks} is an approach to inductive transfer that improves generalization by using the domain information contained in the training signals of related tasks as an inductive bias. It does this by learning tasks in parallel while using a shared representation; what is learned for each task can help other tasks be learned better. Many MTL approaches have shown their efficacy and remarkable performance in many areas such as computer vision \cite{kokkinos2017ubernet, chennupati2019auxnet}, natural language processing \cite{subramanian2018learning, collobert2008unified}, and speech recognition \cite{huang2015rapid}. 

\begin{figure}
\begin{center}
    \includegraphics[width = 0.36\textwidth]{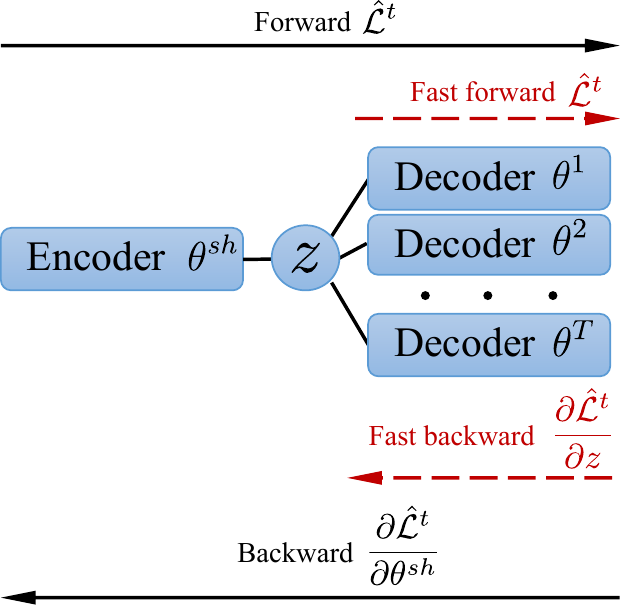}
    \caption{\textit{Solid:} Classical backtracking line search requires several forward and backward steps for choosing stepsize. \textit{Dashed:} We propose to do a fast approximate backtracking in the latent space in order to avoid full backpropagation.}
    \label{fig:fast_bt_scheme}
\end{center}
\end{figure}

\begin{figure*}
\begin{center}
    \includegraphics[width = \textwidth]{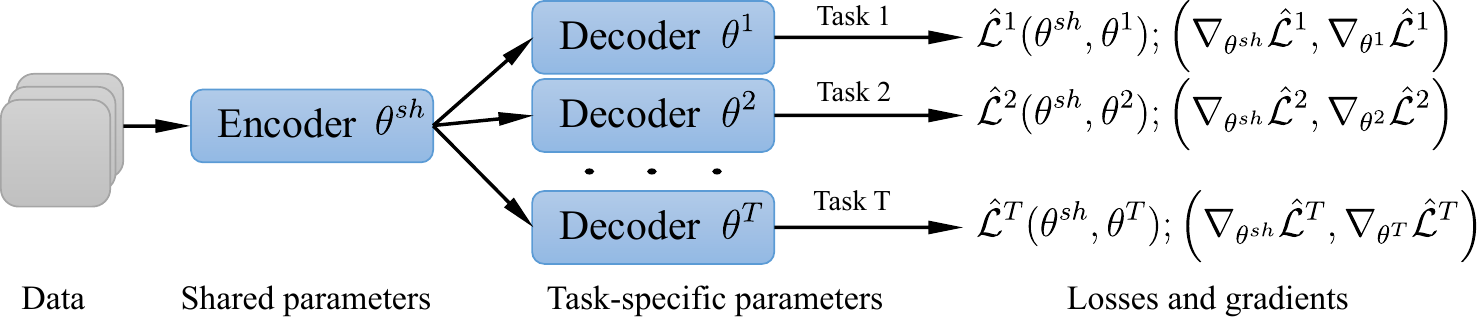}
    \caption{Example of hard parameter sharing architecture}
    \label{fig:harg}
\end{center}
\end{figure*}

A neural network is a state-of-the-art model for solving various machine learning tasks, and multi-task learning is not an exception. Optimization of neural networks often utilizes gradient descent methods. Several methods were proposed to adapt gradient descent for multi-task problems. The first approach is scalarization \cite{johannes1984scalarization}~--- a reduction of a multi-task problem to a single-task problem. The classical scalarization method is weighting~--- treating a multi-task problem as a weighted combination of single-task problems. This approach heavily relies on proper weight selection. Another approach that has no weight selection problem is treating multi-task learning as a multiobjective optimization problem. In this case, we search for direction minimizing all functions. Nevertheless, a problem of choosing step size arises in this approach.

Classical adaptive step size methods like ADAM \cite{kingma2014adam}, Nesterov Accelerated Gradient \cite{nesterov1983method}, RMSProp \cite{tieleman2012lecture}, SAG \cite{schmidt2017minimizing} can be applied for multi-task learning, and they perform well in practice but often comes with little theoretical guarantees for convergence. Using line search methods as backtracking \cite{armijo1966minimization, wolfe1969convergence, vaswani2019painless, vaswani2020adaptive} golden search and parabolic interpolation has theoretical guarantees for convergence but requires additional function calls, which in neural network case has a hefty cost. Especially it is impractical in multi-task learning when there are several task-specific modules, and each function call requires a forward step through a huge encoder.

Addressing the well-known problem of choosing learning rate, we propose a novel optimization idea, which is based on structural properties of multi-task models with "hard shared parametrization" \cite{caruana1997multitask}. Our idea is to do only task-specific model forward passes in backtracking line search (see Figure \ref{fig:fast_bt_scheme}). At the same time, instead of computing new encoder output, we use linear approximation, as it was suggested in \cite{sener2018multi}. It dramatically reduces line search time cost as the encoder containing most cost is inactive. That is why we name this idea \textit{Fast Backtracking Line Search} (FBLS) later in the text. 

We tested our method on Multi-MNIST \cite{sabour2017dynamic}, CIFAR-10 \cite{krizhevsky2014cifar} and Cityscapes \cite{cordts2016cityscapes} with fast backtracking, classical backtracking and constant learning rate stochastic gradient descent (SGD).  Also, we compared these algorithms on the BERT \cite{devlin2018bert} model as an encoder in an artificial multi-task learning setting. We measured average epoch time for the algorithms on the subset of STSB dataset \cite{cer2017semeval}. The outcomes of our work are the following:
\begin{itemize}
    \item We have proposed the Fast Backtracking Line Search method for multi-task learning and have proven its convergence to the Pareto stationary point.
    \item We have compared the proposed approach numerically with classical backtracking line search, SGD with a constant learning rate, and SGD in the latent space. The larger encoder is relative to the decoder size, the greater efficiency the proposed approach shows. It is worth noting that at a significant decrease of epoch time, in many experiments, final quality after the same number of epochs was not worse than SGD.
\end{itemize}

The work is organized as follows. In the first part, we introduce necessary preliminaries, which are specific for multiobjective optimization. In the second part, we introduce Fast Backtracking Line Search and provide the convergence theorem for it. Then, we show the results of our experiments for computer vision (CV) and natural language processing (NLP) models. In the end, we provide a review of related works in the field, followed by a conclusion. 

\section{Preliminaries: Problem and Notation}
\subsection{Notation.}
Consider a multi-task model with encoder-decoder structure (see Figure \ref{fig:harg}). Let $\bm{z}(\bm{x}, \bm{\theta}^{sh})$ be an encoder model. Let $\bm{f}^t(\bm{z}(\bm{x}, \bm{\theta}^{sh}), \bm{\theta}^{t})~=~\bm{f}^t(\bm{z}, \bm{\theta}^{t})$ be  task-specific decoder for task $t$. We use $\{ \bm{x}_i, \bm{y}_i^1, \ldots, \bm{y}_i^T\}_{i=1}^n$ to refer to a batch of objects and their corresponding labels. 

Vectors of parameters $\bm{\theta}^{sh}$ lie in the parameter space, while the decoder outputs $\bm{z}(\bm{x}, \bm{\theta}^{sh})$ belongs to the latent space. Let $\hat{\gL}^t(\bm{\theta}^{sh}, \bm{\theta}^t) = \mathbb{E}_{\bm{x}} \left[ \gL^t\left(\bm{z}(\bm{x}, \bm{\theta}^{sh}), \bm{f}(\bm{z}, \bm{\theta}^{t})\right)\right]$ be a loss function for task $t$. Let $\bm{d_z}$ and $\bm{d_{sh}}$ be some vectors in latent space and parameter space, respectively, which satisfies the following conditions:
    \[
        \forall t \in \{1, \ldots, T\}: \quad (\nabla_{\bm{z}} \hat{\gL}^t)^{\top} \bm{d}_z < 0
    \]
    \[
        \forall t \in \{1, \ldots, T\}: \quad (\nabla_{\bm{\theta}^{sh}} \hat{\gL}^t)^{\top} \bm{d}_{sh} < 0
    \]
So, applying gradient step with direction $\bm{d_z}$ in latent space or with direction $\bm{d_{sh}}$ in parameter space decreases all loss functions. 

\subsection{Problem formulation.}Consider multi-task learning problem with $T$ tasks:

\begin{equation}
\label{fastbtmtl:eq:pareto}
\begin{gathered}
\min _{\substack{\bm{\theta}^{sh},\\ \bm{\theta}^{1}, \ldots, \bm{\theta}^{T}}} \bm{\hat{\gL}}\left(\bm{\theta}^{sh}, \bm{\theta}^{1}, \ldots, \bm{\theta}^{T}\right),  \\ 
\bm{\hat{\gL}} = \left(\hat{\gL}^{1}\left(\bm{\theta}^{sh}, \bm{\theta}^{1}\right), \ldots, \hat{\gL}^{T}\left(\bm{\theta}^{sh}, \bm{\theta}^{T}\right)\right)^{\top}
\end{gathered}
\end{equation}

To solve this multi-task learning problem, we treat it as a multiobjective optimization problem and introduce partial order - Pareto dominance. 

\renewcommand\thesection{1}

\begin{definition}{(Pareto dominance).} We say, that a vector $\bm{\theta}_1$ is dominated by vector $\bm{\theta}_2$ iff:
\[
\forall t \in \{1,\ldots, T\}: \ 
 \hat{\gL}^t(\bm{\theta}_2) \leq \hat{\gL}^t(\bm{\theta}_1) 
\]
\[
\exists i: \hat{\gL}^i(\bm{\theta}_2) < \hat{\gL}^i(\bm{\theta}_1)
\]

\end{definition}

In our case, we treat a multiobjective optimization as a problem of finding a point that is not Pareto dominated --- a Pareto optimal point. 

\begin{definition}{(Pareto optimality).}
We say, that a vector $\hat{\bm{\theta}}$ in problem \ref{fastbtmtl:eq:pareto}  is Pareto optimal if and only if there $\nexists \ \bm{\theta}$ such that Pareto dominate $\hat{\bm{\theta}}$.
\end{definition}

In particular, if a point $\hat{\bm{\theta}}$ is not Pareto optimal, there is at least one other point that should be better. \cite{boyd2004convex}

Finding a Pareto optimal point could be a very challenging problem. However, we can try to satisfy the necessary condition for Pareto optimality - Pareto stationarity.

\begin{definition}{(Pareto stationarity).}
$\bar{\bm{\theta}}$ is a Pareto stationary point for problem 1 iff there exist $\alpha^{1}, \ldots, \alpha^{T} \geq 0$ such that $\sum_{t=1}^{T} \alpha^{t}=1$ and $\sum_{t=1}^{T} \alpha^{t} \nabla_{\bm{\theta}^{sh}} \hat{\gL}^{t}\left(\bar{\bm{\theta}}^{s h}, \bar{\bm{\theta}}^{t}\right)=0$ and $\forall \ t, \nabla_{\bm{\theta}^{t}} \hat{\gL}^{t}\left(\bar{\bm{\theta}}^{s h}, \bar{\bm{\theta}}^{t}\right)=0$
\end{definition}

Pareto stationarity is a necessary condition for Pareto optimality \cite{sener2018multi, desideri2012multiple}. 

To find Pareto stationary point one can use gradient methods such as MGDA \cite{sener2018multi}, EDM \cite{katrutsa2020follow} or PCGrad \cite{yu2020gradient}, which are forming decreasing direction, combining several gradients for different tasks. They provide a minimizing direction that minimizes all functions. Next, we have to choose a step size for going through this minimizing direction. To the best of our knowledge, there is no theoretical evidence for using constant step size or adaptive step-size gradient methods in smooth multiobjective minimization. There is an important result about convergence with steepest descent-like and Armijo-like line search \cite{fliege2000steepest} for gradient descent. 

The main idea of line search is to find step size $\eta$ which satisfies the following condition $\forall t\in \{1\ldots T\}$:

\begin{equation}
     \hat{\gL}^t(\bm{\theta}-\eta \bm{d}) < \hat{\gL}^t(\bm{\theta}).
    \label{condition_1}
\end{equation}

\begin{algorithm}[t]
\caption{Backtracking Line Search (BLS)}
\label{fastbtmtl:alg:back}
\begin{algorithmic}[1]
\REQUIRE{Loss functions $\hat{\gL}_1, \ldots, \hat{\gL}_T$, initial point $[\bm{\theta}^{sh}, \bm{\theta}^1, \ldots, \bm{\theta}^{T}]$, task-specific gradients $\nabla_{\bm{\theta}^{t}}\hat{\gL}^t$, minimizing direction $\bm{d}_{sh}$ }, hyperaparameters $\beta, \gamma, lr_{ub}$
\REPEAT
\STATE{$\eta := lr_{ub}$}
\REPEAT
\STATE{\fbox{$\bm{\tilde{\theta}}^{sh} \leftarrow \bm{\theta}^{sh} - \eta\cdot \bm{d}_{sh}$}}
\FOR{$t \leftarrow 1 \ \text{to} \ T$}
\STATE{$\tilde{\bm{\theta}^{t}} \leftarrow \bm{\theta}^{t} - \eta\cdot \nabla_{\bm{\theta}^{t}}\hat{\gL}^t$}
\ENDFOR
\STATE{$\eta \leftarrow \gamma \cdot \eta$}
\UNTIL{Armijo condition \ref{armijobs} $(\beta, \eta)$}
\FOR{$t \leftarrow 1 \ \text{to} \ T$}
\STATE{$\bm{\theta}^{t}_{new} \leftarrow \bm{\tilde{\theta}}^{t}$}
\ENDFOR
\STATE{\fbox{$\bm{\theta}^{sh}_{new} \leftarrow \bm{\tilde{\theta}}^{sh}$}}
\UNTIL{Stopping criterion}
\end{algorithmic}
\end{algorithm}

\begin{algorithm}[t]
\caption{Fast Backtracking Line Search (FBLS)}
\label{fastbtmtl:alg:main}

\begin{algorithmic}[1]
\REQUIRE{Loss functions $\hat{\gL}_1, \ldots, \hat{\gL}_T$, initial point $[\bm{\theta}^{sh}, \bm{\theta}^1, \ldots, \bm{\theta}^T]$, task-specific gradients $\nabla_{\bm{\theta}^t}\hat{\gL}^t$, minimizing direction $\bm{d}_z$, hyperaparameters $\beta, \gamma, lr_{ub}$ }
\REPEAT
\STATE{$\eta := lr_{ub}$}
\REPEAT
\STATE{\fbox{$\bm{z} \leftarrow \bm{z} - \eta\cdot \bm{d}_z$}}
\FOR{$t \leftarrow 1 \ \text{to} \ T$}
\STATE{$\tilde{\bm{\theta}^{t}} \leftarrow \bm{\theta}^{t} - \eta\cdot \nabla_{\bm{\theta}^{t}}\hat{\gL}^t$}
\ENDFOR
\STATE{$\eta \leftarrow \gamma \cdot \eta$}
\UNTIL{Armijo condition \ref{armijofbs} $(\beta, \eta)$}
\FOR{$t \leftarrow 1 \ \text{to} \ T$}
\STATE{$\bm{\theta}^{t}_{new} \leftarrow \tilde{\bm{\theta}}^{t}$}
\ENDFOR
\STATE{\fbox{$\bm{\theta}^{sh}_{new} \leftarrow \bm{\theta}^{sh} - \eta \cdot \left(\frac{\partial \bm{z}}{\partial \bm{\theta}^{sh}}\right)^{\top} \bm{d}_z$}}
\UNTIL{Stopping criterion}
\end{algorithmic}
\end{algorithm}

This condition means that at each step, all functions are decreasing. If, in addition, the functions $\gL$ are bounded from below, the exact line search (like in steepest descent) guarantees convergence to a Pareto stationary point. In our paper, we consider only backtracking line search (BLS), but the ideas work for an arbitrary line search method as well. The idea of backtracking is to reduce the step size by parameter $\gamma$ until the line search condition is satisfied. Instead of using the beforementioned condition, it is usually replaced by Armijo rule \cite{armijo1966minimization} with parameter $\beta \in [0, 1]$, $\forall t\in \{1\ldots T\}$:

\begin{equation}
     \hat{\gL}^t(\bm{\theta}-\eta \bm{d}) \leq \hat{\gL}^t(\bm{\theta}) - \eta\beta \left(\frac{\partial \hat{\gL}^t}{\partial \bm{\theta}}\right)^\top \bm{d}
    \label{condition_2}
\end{equation}

So, to find the appropriate step size, we should use a line search. Nevertheless, it is computationally ineffective as on each iteration, we should perform forward and backward pass through the entire encoder and all decoders. We repeat this procedure till the Armijo rule has been fulfilled. It may take several steps and cause heavy wastage of computational resources since we do several forward and backward passes for making one gradient step.

The problem of substantial computational overhead for backtracking may be especially relevant for models with huge encoders (relatively to the decoder size) in terms of the number of parameters. An example of such a model is the Transformer \cite{vaswani2017attention, liu2019roberta, dosovitskiy2020image, hu2021transformer, zhang2021bmt}, which has become the de-facto standard for NLP tasks. 

Generally, in NLP, a model has a heavy encoder to produce consistent word embeddings, while the decoder could be much smaller. In this case minimizing using encoder model can significantly improve computational speed. Therefore, in such models, latent space optimization could be of particular interest, as we only once use encoder just to calculate latent representation. In the following part, we propose a fast line search in the form of Fast Backtracking Line Search (FBLS) that deals with this problem.

\section{Method}

In this section, we propose an idea of using latent space for finding step size. The optimization in latent space allows not to use additional forward and backward passes through the encoder. To establish the capabilities of the method, we experimented with the line search backtracking algorithm.

Consider gradients in the parameter and latent spaces: $\nabla_{\bm{\theta}^{sh}} \gL^t, \nabla_{\bm{z}} \gL^t$. To
obtain minimization direction $\bm{d}_{sh}$ or $\bm{d}_z$ one can use
PCGrad \cite{yu2020gradient}, or EDM \cite{katrutsa2020follow}, or MGDA \cite{desideri2012multiple, sener2018multi}.

As we get the minimization direction, one can consider the classical backtracking approach for multi-task learning. In this case, Armijo-rule can be represented as $\forall t \in \{1\ldots T\}$:

\begin{equation}
\begin{split}
 &\hat{\gL}^t(\bm{\theta}^{sh}-\eta \bm{d}_{sh}, \bm{\theta}^{t}-\eta\nabla_{\bm{\theta}^{t}} \hat{\gL}^t) \leq \\ &\leq \hat{\gL}^t-\eta\beta \left\|\frac{\partial \hat{\gL}^t}{\partial \bm{\theta}^{t}}\right\|^2 - \eta\beta \left(\frac{\partial \hat{\gL}^t}{\partial \bm{\theta}^{sh}}\right)^\top \bm{d}_{sh}
 \end{split}
\label{armijobs}
\end{equation}

The complete classical backtracking algorithm is presented as Algorithm~\ref{fastbtmtl:alg:back}.

One can see that in the classical approach, one needs to update shared parameters several times on every iteration. We use the following local linear approximation:

\begin{equation}
\bm{z}(\bm{\theta}^{sh} - \eta {\bm{d}_{sh}}) \approx \bm{z} - \eta \bm{d}_z
\label{eq:enc_linear_approx}
\end{equation}

Then, one can get a new Armijo rule:
\begin{equation}
    \begin{split}
    &\hat{\gL}^t(\bm{z}-\eta \bm{d}_z, \bm{\theta}^{t} - \eta\nabla_{\bm{\theta}^{t}} \hat{\gL}^t) \leq \\
    &\leq \hat{\gL}^t - \eta\beta \left\|\frac{\partial \hat{\gL}^t}{\partial \bm{\theta}^{t}}\right\|^2 - \eta\beta \left(\frac{\partial \hat{\gL}^t}{\partial \bm{z}}\right)^{\top} \bm{d}_z
    \end{split}
    \label{armijofbs}
\end{equation}

The complete fast backtracking algorithm is presented as Algorithm~\ref{fastbtmtl:alg:main}. In this Armijo rule, we do not have updates of encoder parameters $\bm{\theta}^{sh}$. So we can keep the encoder frozen during the step size search stage.

In the following theorem, we prove that the new algorithm converges to Pareto stationary point.   

\begin{restatable}{theorem}{theconv}\label{armijo:theorem}
	Suppose loss functions $\hat{\gL}^t$ are continuously differentiable and bounded below. Then, every accumulation point of sequence $\bm{x_k} = [\bm{\theta^{sh}_k}, \bm{\theta^1_k}, \ldots, \bm{\theta}^{T}_k]_{k=1}^{\infty}$ produced by Algorithm \ref{fastbtmtl:alg:main} with following Armijo-rules is a Pareto stationary point $\forall t \in \{ 1 \ldots T \}$:
	\begin{align}
		&\hat{\gL}^t(\bm{z}-\eta \bm{d}_z, \bm{\theta}^{t} - \eta\nabla_{\bm{\theta}^{t}} \hat{\gL}^t) \leq \nonumber\\
		\ldots &\leq \hat{\gL}^t - \eta\beta \left\|\frac{\partial \hat{\gL}^t}{\partial \bm{\theta}^{t}}\right\|^2 \label{armijo1}\\
		\ldots & \leq \hat{\gL}^t - \eta\beta \left(\frac{\partial \hat{\gL}^t}{\partial \bm{\theta}^{t}}\right)^\top \bm{d}_z \label{armijo2} \\
		\ldots & \leq \hat{\gL}^t - \eta\beta \left(\frac{\partial \hat{\gL}^t}{\partial \bm{\theta}^{t}}\right)^\top \bm{d}_z - \eta\beta \left\|\frac{\partial \hat{\gL}^t}{\partial \bm{\theta}^{t}}\right\|^2 \label{armijo3}
	\end{align}
\end{restatable}

We prove Theorem \ref{armijo:theorem} in Appendix. Indeed, using a linear approximation of the encoder \eqref{eq:enc_linear_approx} instead of its full update allows us to keep the convergence properties for Armijo-like line search. In the next part, we conduct analysis FBLS, BLS, SGD, MGDA-UB methods.

\section{Experiments
}
The experiments can be divided into 4 groups by data type: MultiMNIST, CIFAR-10, Cityscapes, STS benchmark for BERT model.
\begin{table}[h!]
{\centering
\begin{tabular}{cccc}
\toprule
\textbf{Alg / Data} & \textbf{MultiMNIST}      & \textbf{CIFAR-10}  & \textbf{w/o} $\bm{\eta}$?\\ 
\midrule
SGD         & 14.3 (100 \%) & 550 (100 \%) & $\bf{x}$ \\ 
BLS         & 19.5 (136 \%) & 650 (118 \%) & \checkmark \\ 
MGDA-UB     & 13.6 (95  \%) & 80  (15  \%) & $\bf{x}$ \\ 
FBLS (Ours) & 14.3 (100 \%) & 85  (15  \%) & \checkmark \\ 
\bottomrule
\end{tabular}
\caption{Average epoch time in seconds for different algorithms on MultiMNIST and CIFAR-10 datasets. By "w/o~$\eta$?" we mean the ability to run an algorithm without selecting a learning rate $\eta$. The lower, the better. We took SGD as the baseline - 100 \%.}
\label{fastbtmtl:table:times}}
\end{table}

For comparison, we examined fast backtracking, classical backtracking, classical SGD, and Latent space SGD (MGDA-UB) \cite{sener2018multi}. Experiments were conducted on MultiMNIST, CIFAR-10, and Cityscapes datasets. In practice for Algorithms \ref{fastbtmtl:alg:back}, \ref{fastbtmtl:alg:main}, we have set additional stopping criterion what $\eta$ can't be lower some small $\varepsilon$.  For backtracking algorithms we have chosen learning rate lower bound $\varepsilon = 10^{-10}$, $\beta = 0.1$. The initial upper bound has been set as $1$. At first experiments, we observed that from some iteration step size converged to upper bound. This behavior contradicts the theory of gradient algorithms. We decay step size upper bound by decay rate $= 0.5$ every $N=10$ epochs to overcome this problem. Such version of algorithm is presented as "FBLS+decay" on the figures. Nevertheless, the time comparison, presented on the Table \ref{fastbtmtl:table:times} was averaged over $10$ epochs.

\subsection{MultiMNIST}
\begin{figure}[h!]
    \begin{subfigure}[b]{0.25\textwidth}
            \centering
            \includegraphics[width=\linewidth]{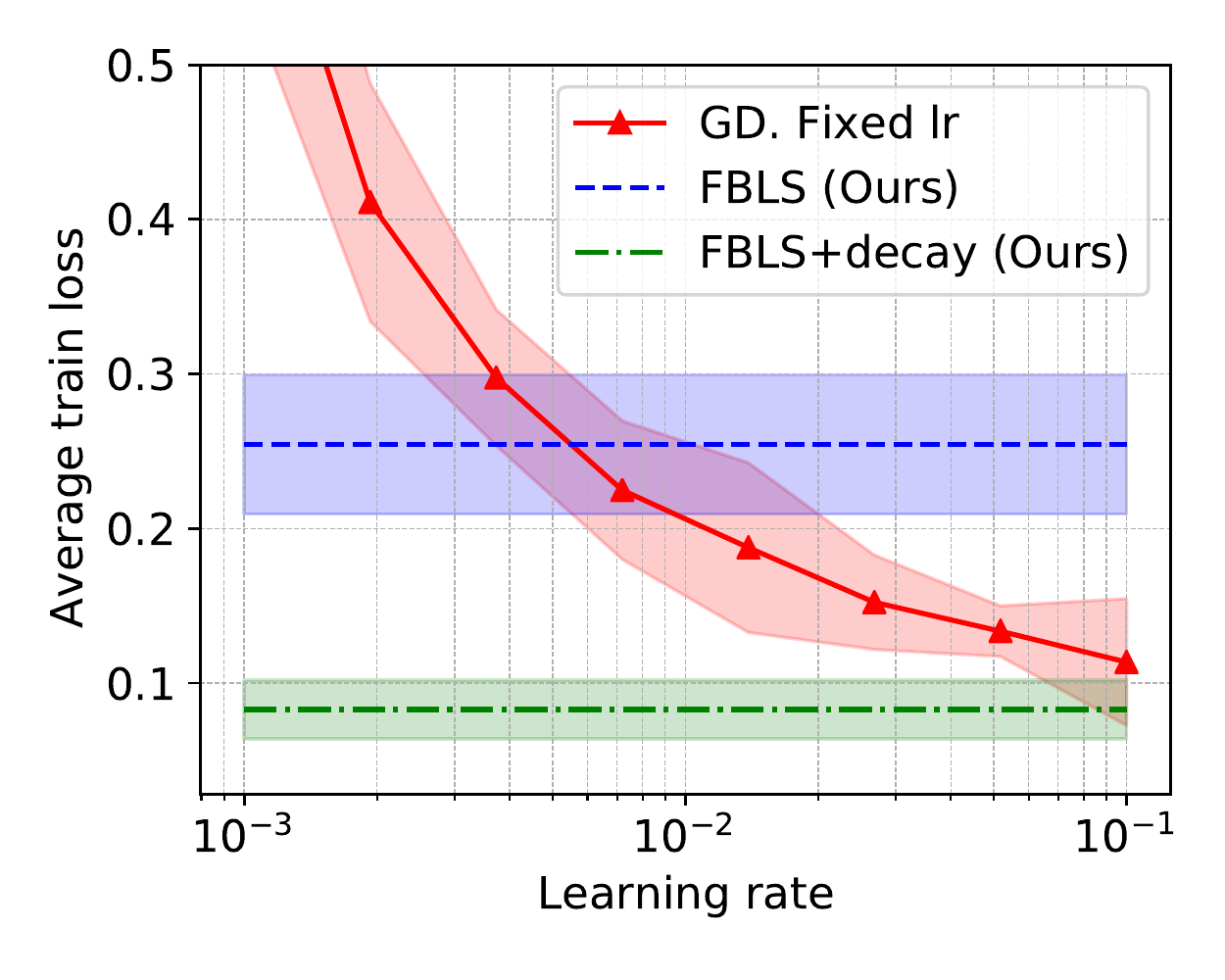}
    \end{subfigure}%
    \begin{subfigure}[b]{0.25\textwidth}
            \centering
            \includegraphics[width=\linewidth]{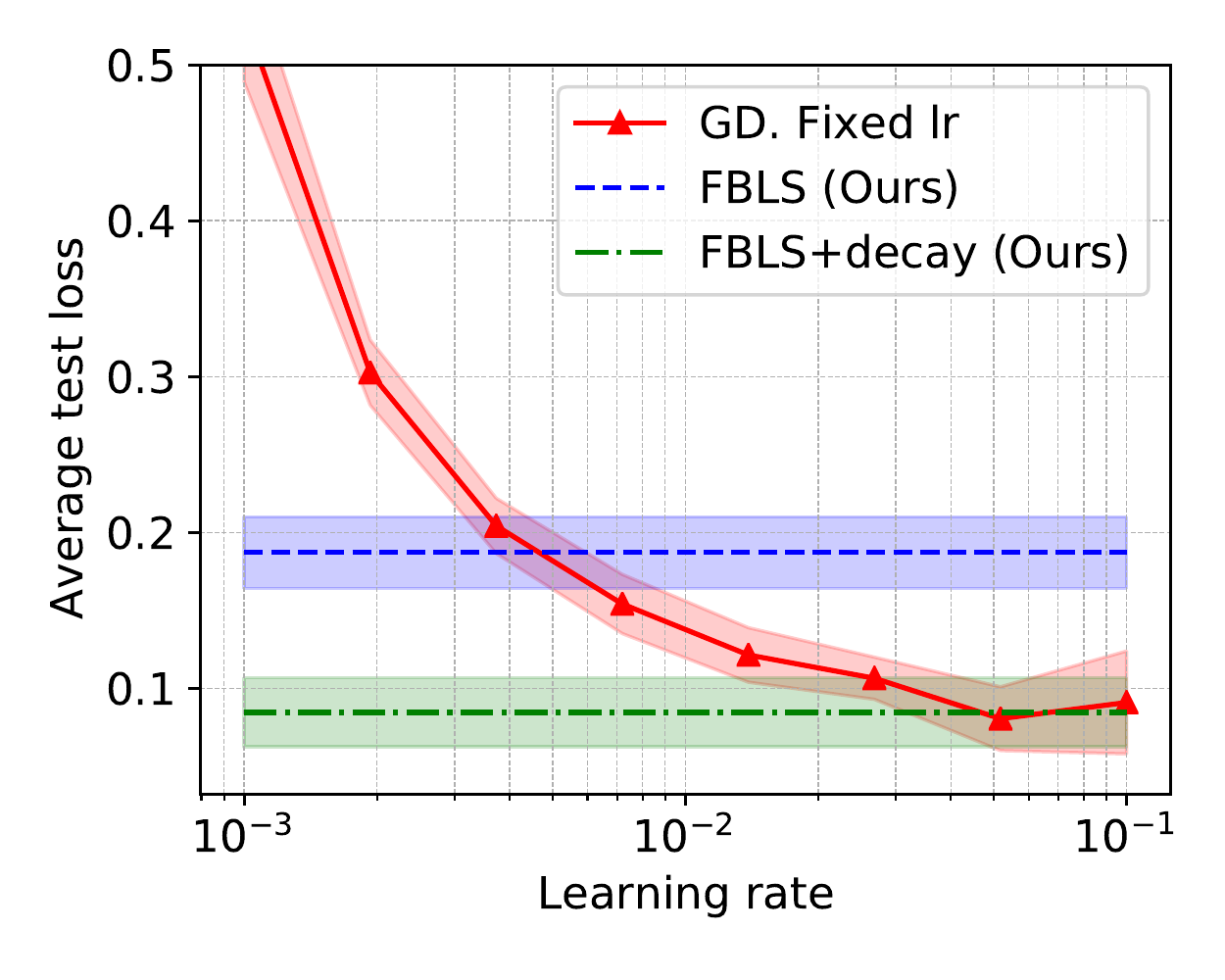}
    \end{subfigure}%
    \caption{Results of Backtracking Line Search perfomance and constant learning rate SGD after 100 epochs on MultiMNIST dataset.}
    \label{fastbtmtl:fig:mmnist}
\end{figure}

On MultiMNIST, we use LeNet-5 architecture. The batch size has been set as 256. The training was conducted for 100 epochs.
For comparison with standard gradient descent, we choose eight step sizes evenly from log-range between -3 and -1. As loss function, we chose cross-entropy loss, and as error, we consider (1 - accuracy). Results of SGD and FBLS were averaged over 3 experiments (Figure \ref{fastbtmtl:fig:mmnist}).

From Figure \ref{fastbtmtl:fig:mmnist} we can see that as good as classical gradient descent. Also, it takes the same time per epoch as classic stochastic gradient descent (see Table. \ref{fastbtmtl:table:times}).

\subsubsection{Classical backtracking}

For comparison of classical and fast backtracking line search we examined their behaviour at different $\beta = \{0.1, 0.2, 0.3, 0.4\}$. Averaged over different $\beta$ results presented in Figure \ref{fastbtmtl:fig:mmnist_fbs_bs}. 

\begin{figure}[h!]
    \begin{subfigure}[b]{0.25\textwidth}
            \centering
            \includegraphics[width=\linewidth]{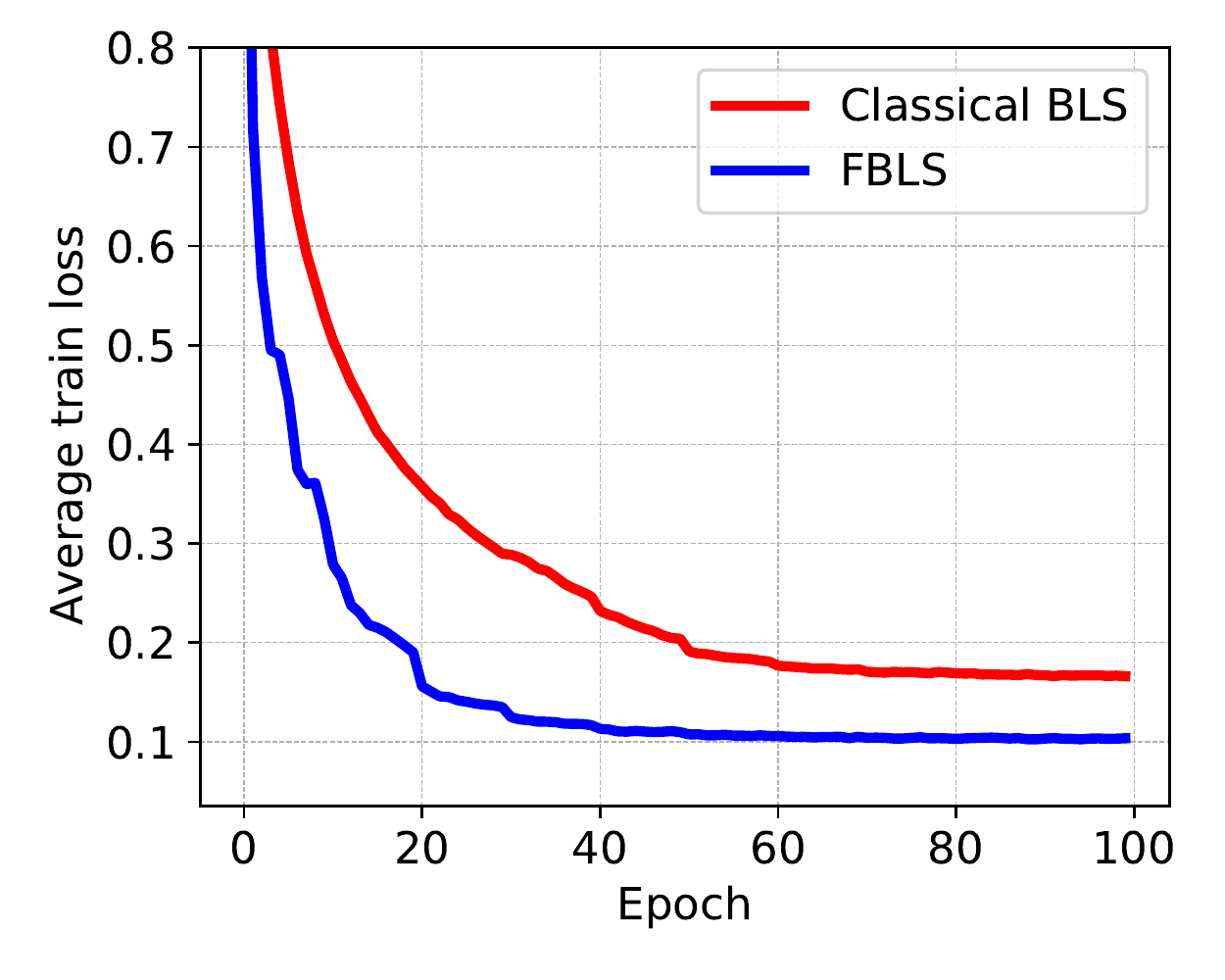}
    \end{subfigure}%
    \begin{subfigure}[b]{0.25\textwidth}
            \centering
            \includegraphics[width=\linewidth]{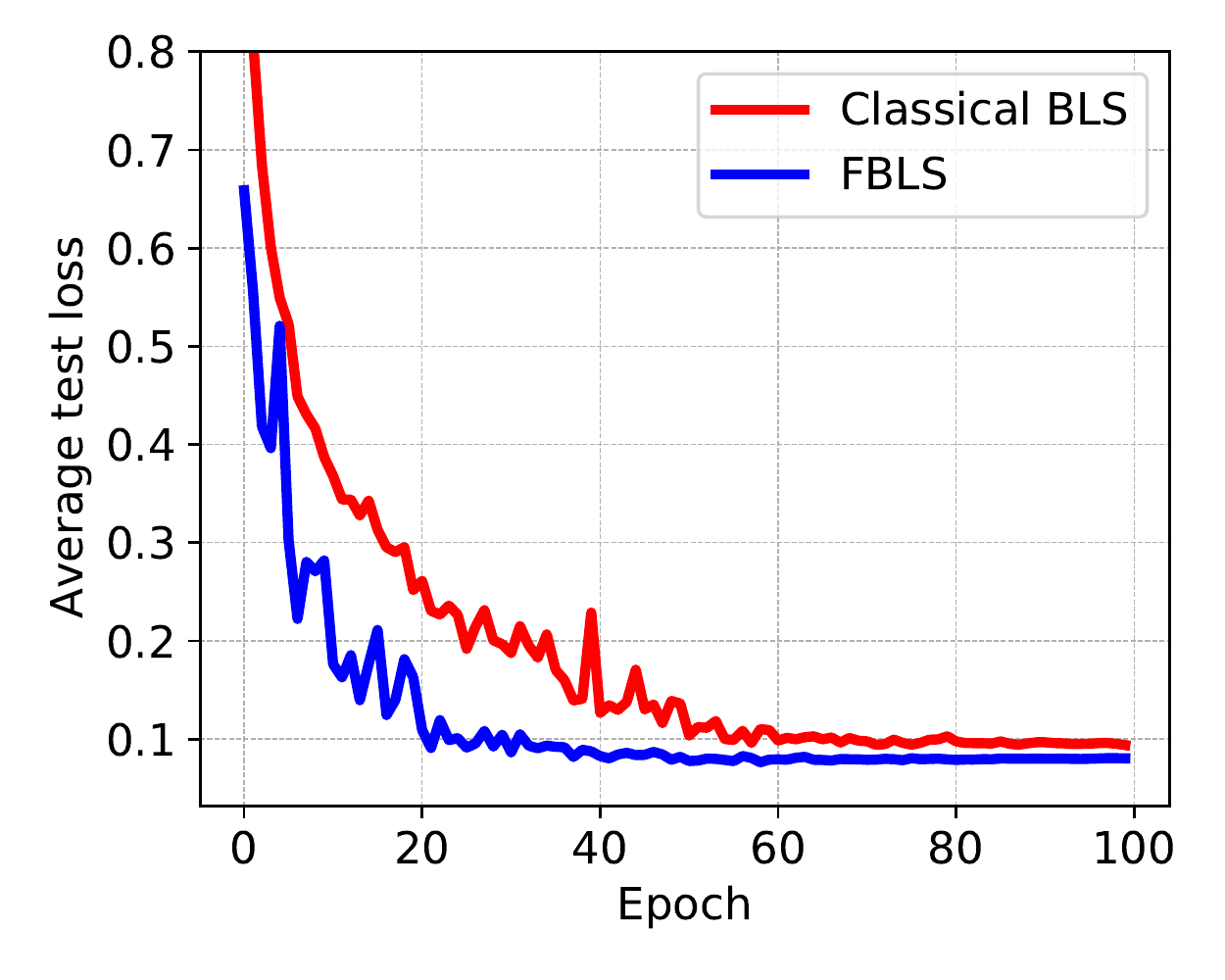}
    \end{subfigure}%
    \caption{Comparison of Classical Backtracking Line Search vs FBLS. Results averaged over four different $\beta$ parameters}
    \label{fastbtmtl:fig:mmnist_fbs_bs}
\end{figure}

From  Figure \ref{fastbtmtl:fig:mmnist_fbs_bs} one can see that fast backtracking overcome classical backtracking both on test and train. Also, fast backtracking is faster than classical backtracking by 36\% on the MultiMNIST (see Table \ref{fastbtmtl:table:times}).

\subsection{CIFAR-10}
\begin{figure}[h!]
    \begin{subfigure}[b]{0.25\textwidth}
            \centering
            \includegraphics[width=\linewidth]{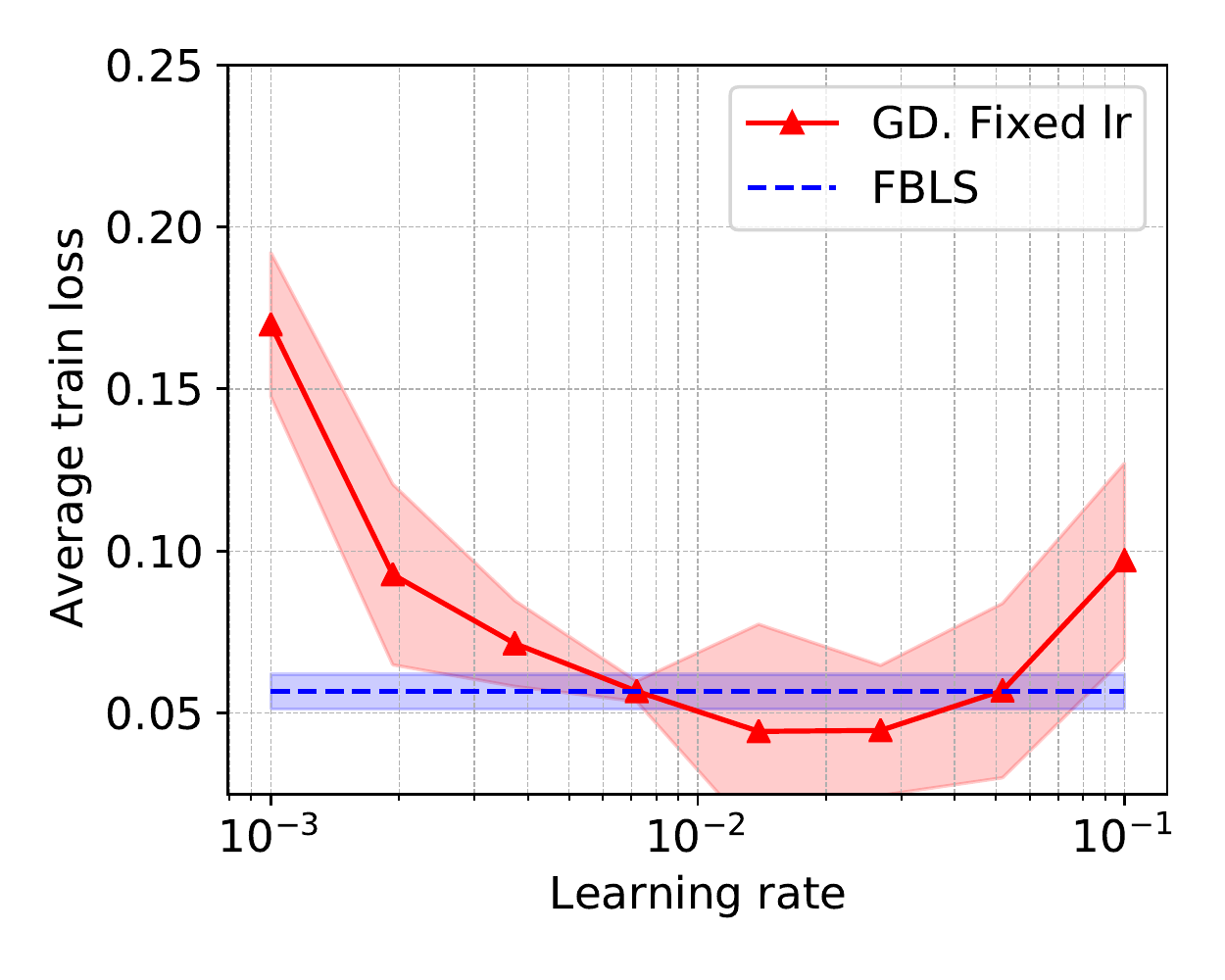}
    \end{subfigure}%
    \begin{subfigure}[b]{0.25\textwidth}
            \centering
            \includegraphics[width=\linewidth]{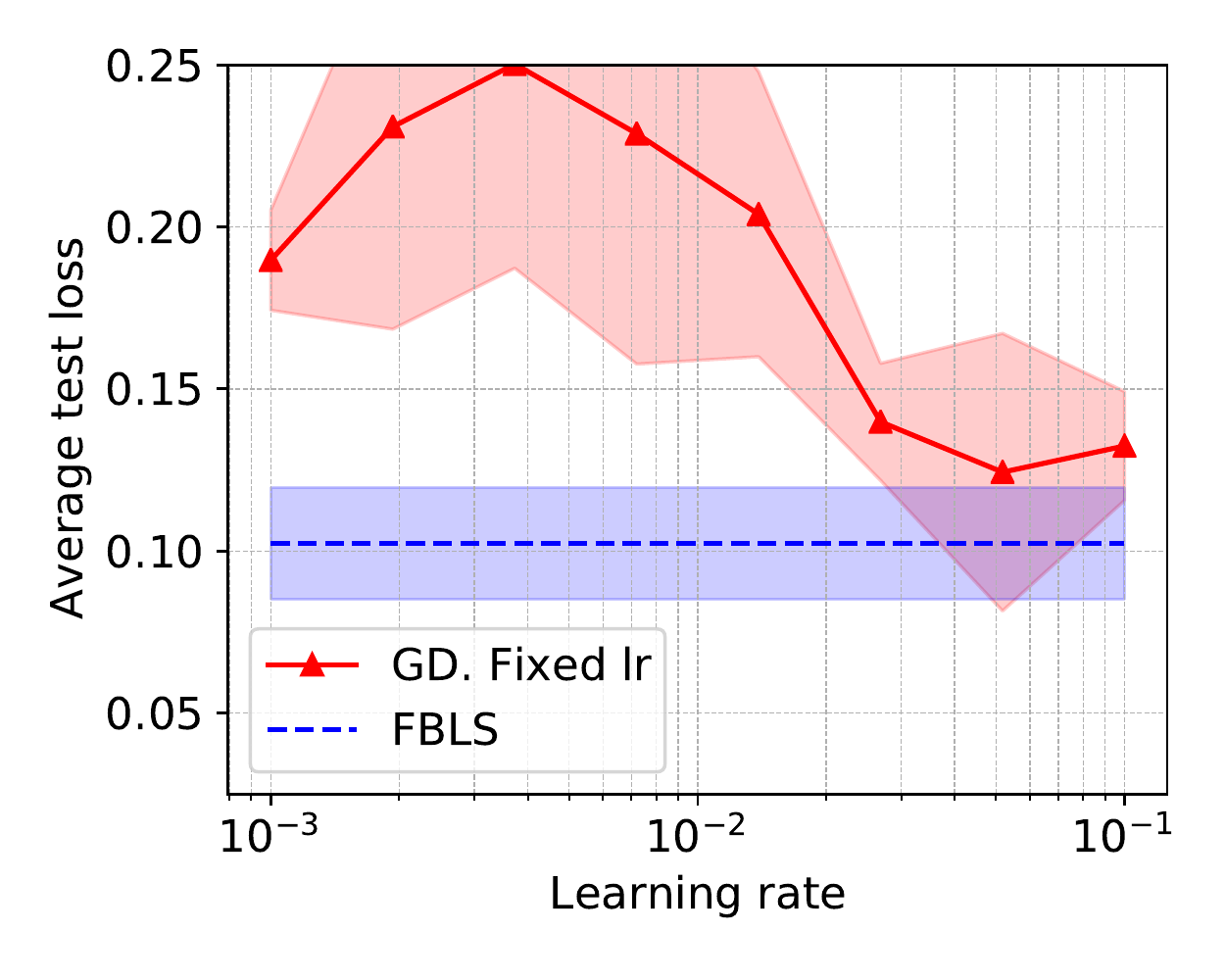}
    \end{subfigure}%
    \caption{Results of Backtracking Line Search perfomance and constant learning rate SGD after 50 epochs on MultiCIFAR10 dataset.}
    \label{fastbtmtl:fig:mcifar}
\end{figure}

\begin{figure*}
    \begin{subfigure}[b]{0.25\textwidth}
            \centering
            \includegraphics[width=\linewidth]{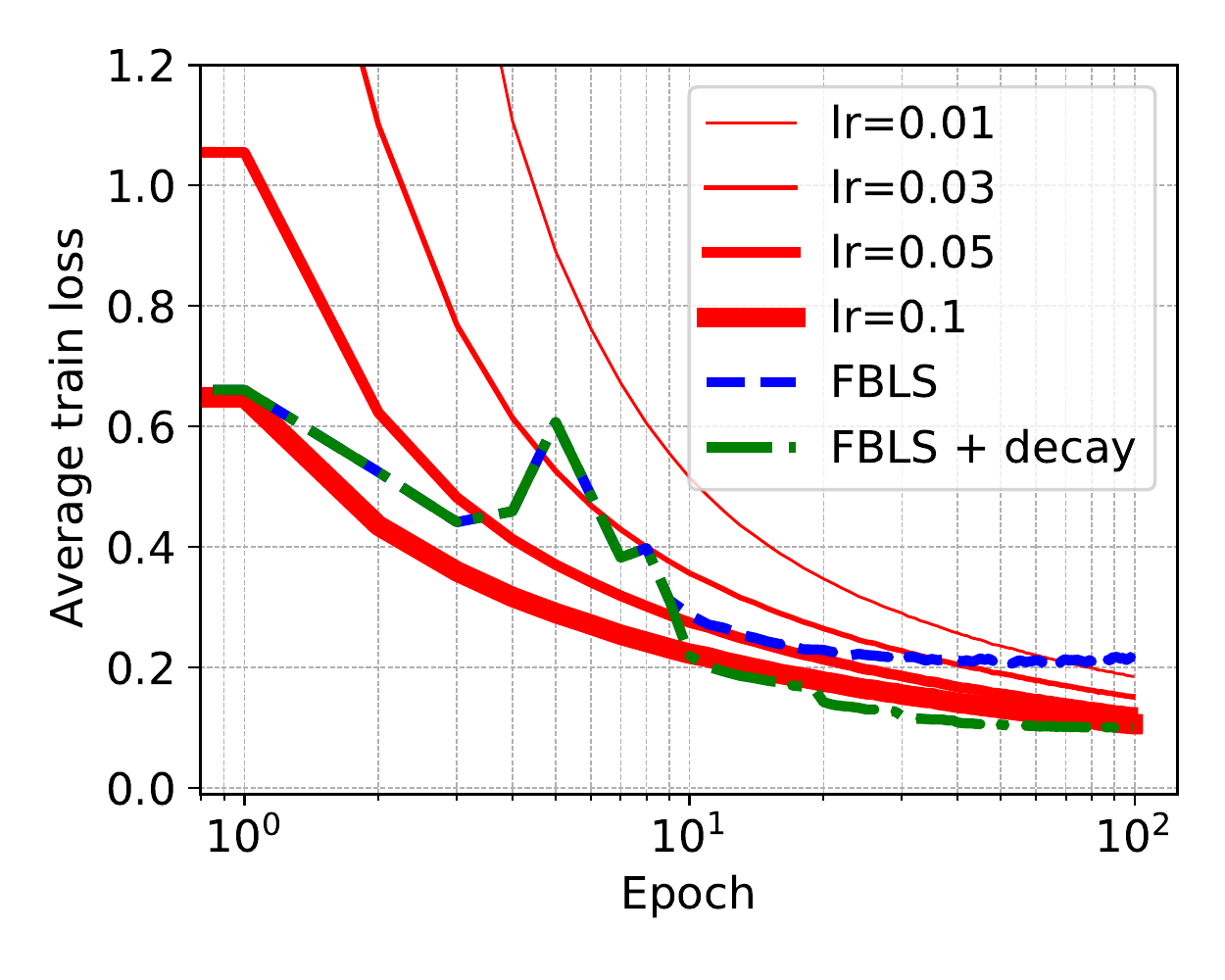}
    \end{subfigure}%
    \begin{subfigure}[b]{0.25\textwidth}
            \centering
            \includegraphics[width=\linewidth]{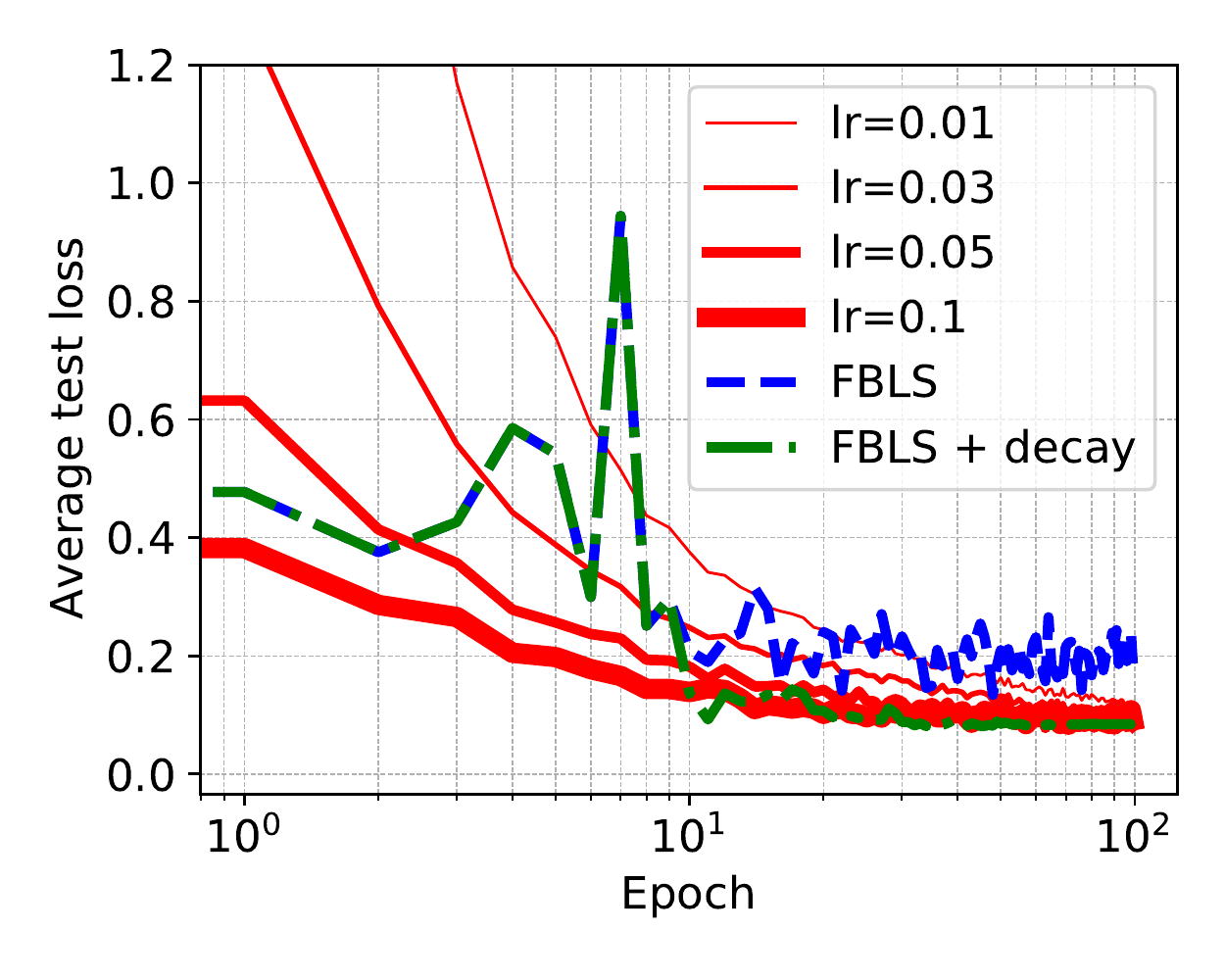}
    \end{subfigure}%
    \begin{subfigure}[b]{0.25\textwidth}
            \centering
            \includegraphics[width=\linewidth]{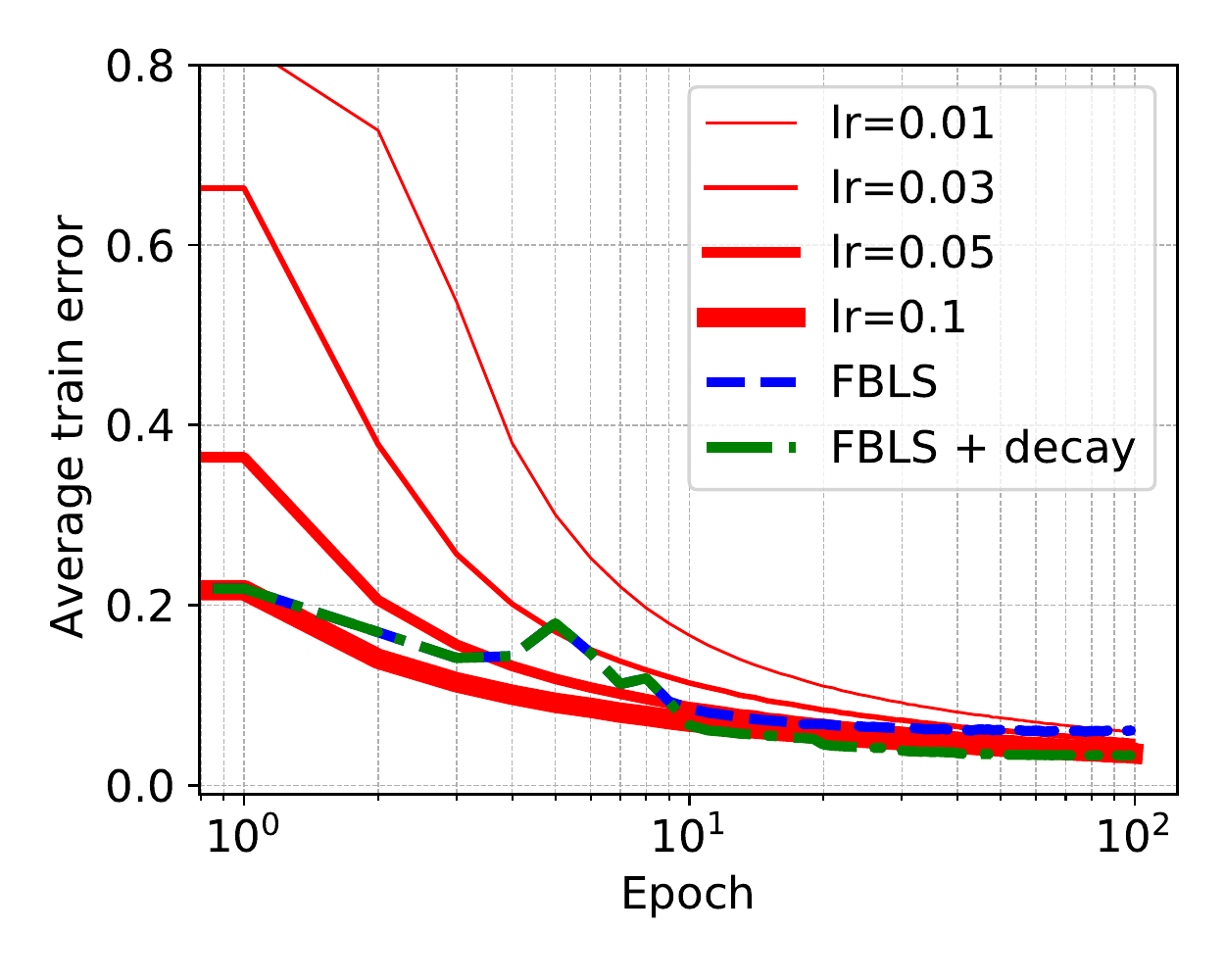}
    \end{subfigure}%
    \begin{subfigure}[b]{0.25\textwidth}
            \centering
            \includegraphics[width=\linewidth]{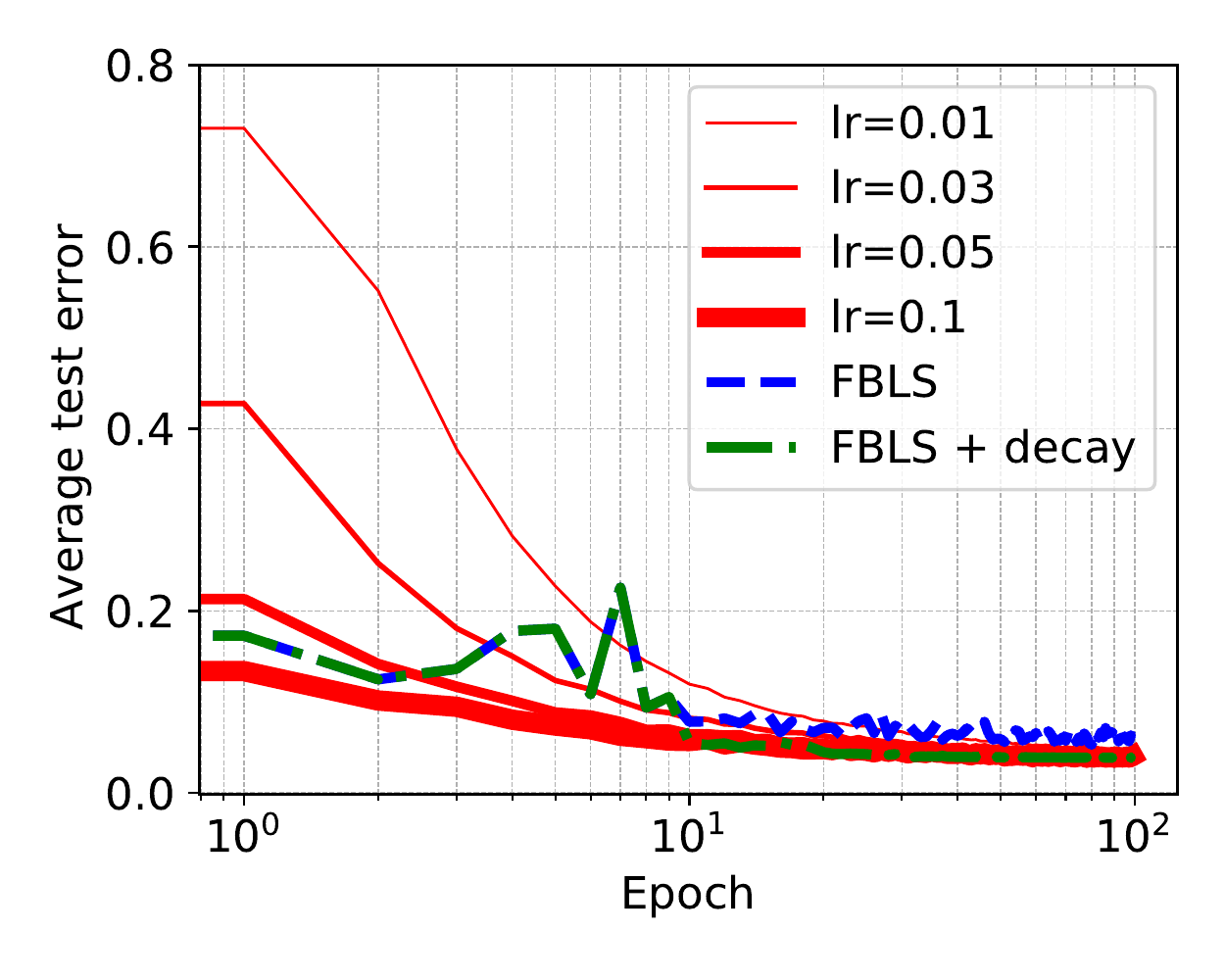}
    \end{subfigure}
    \\
    \begin{subfigure}[b]{0.25\textwidth}
            \centering
            \includegraphics[width=\linewidth]{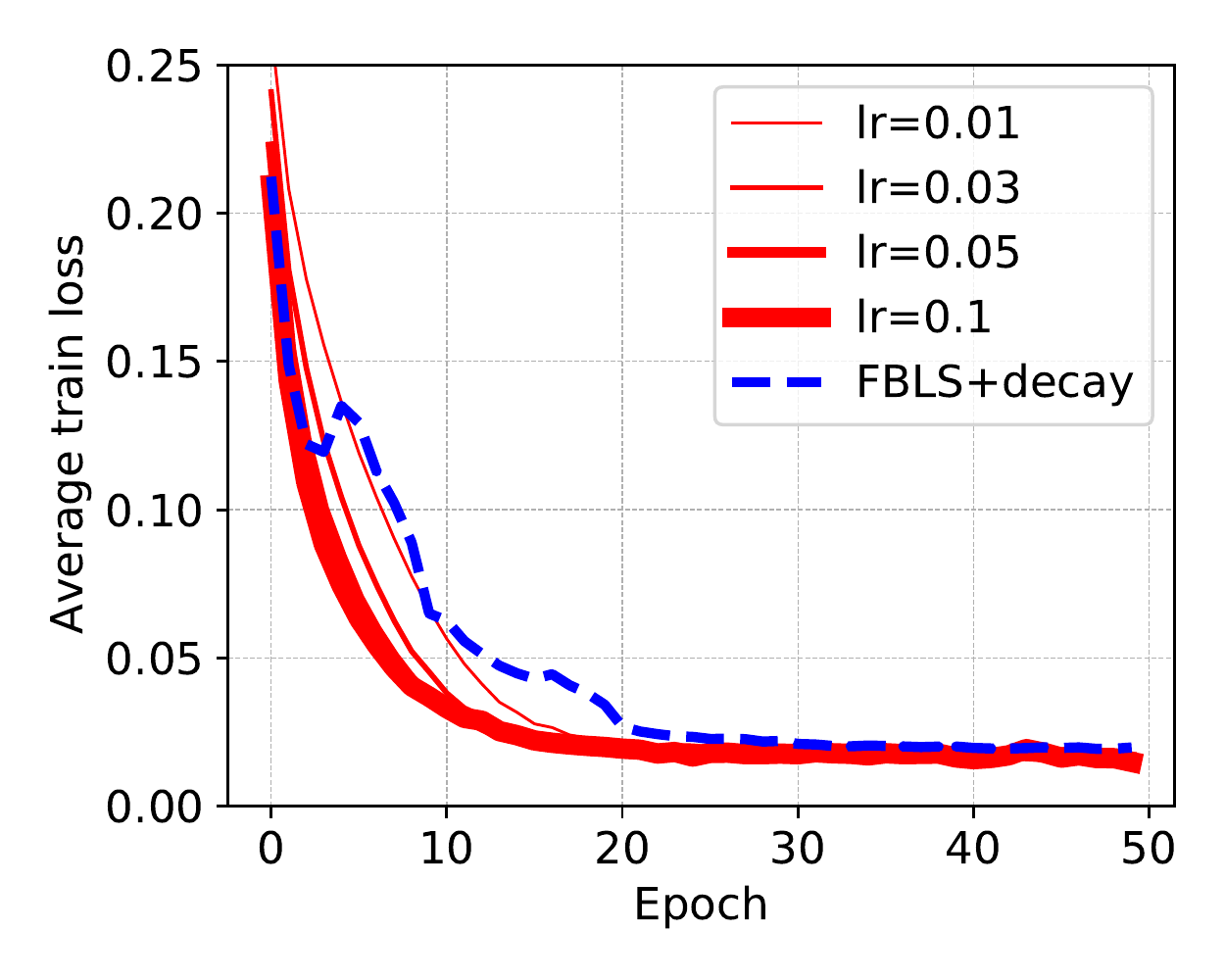}
            \caption{{\small \texttt{Train loss}}}
    \end{subfigure}%
    \begin{subfigure}[b]{0.25\textwidth}
            \centering
            \includegraphics[width=\linewidth]{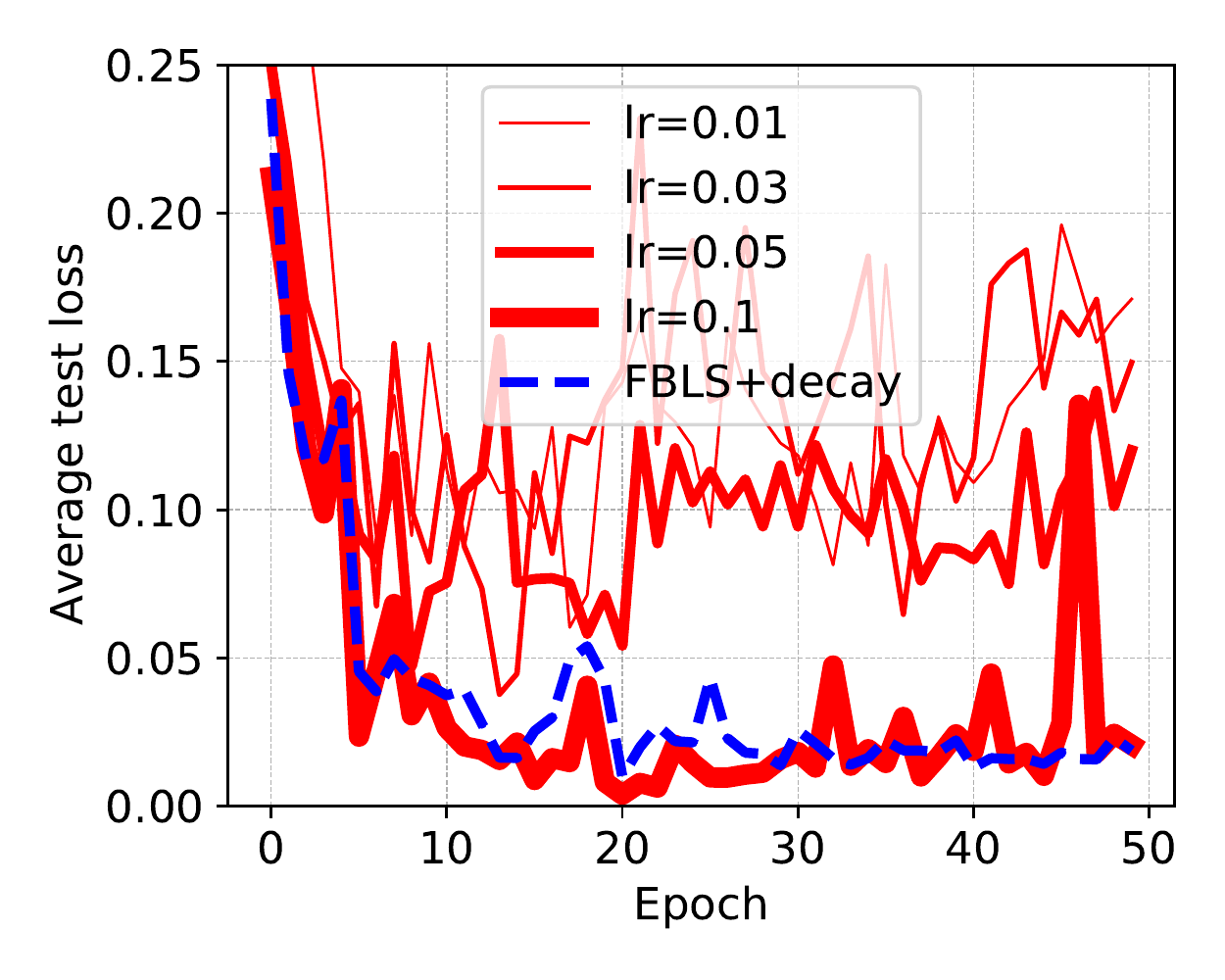}
            \caption{{\small \texttt{Test loss}}}
    \end{subfigure}%
    \begin{subfigure}[b]{0.25\textwidth}
            \centering
            \includegraphics[width=\linewidth]{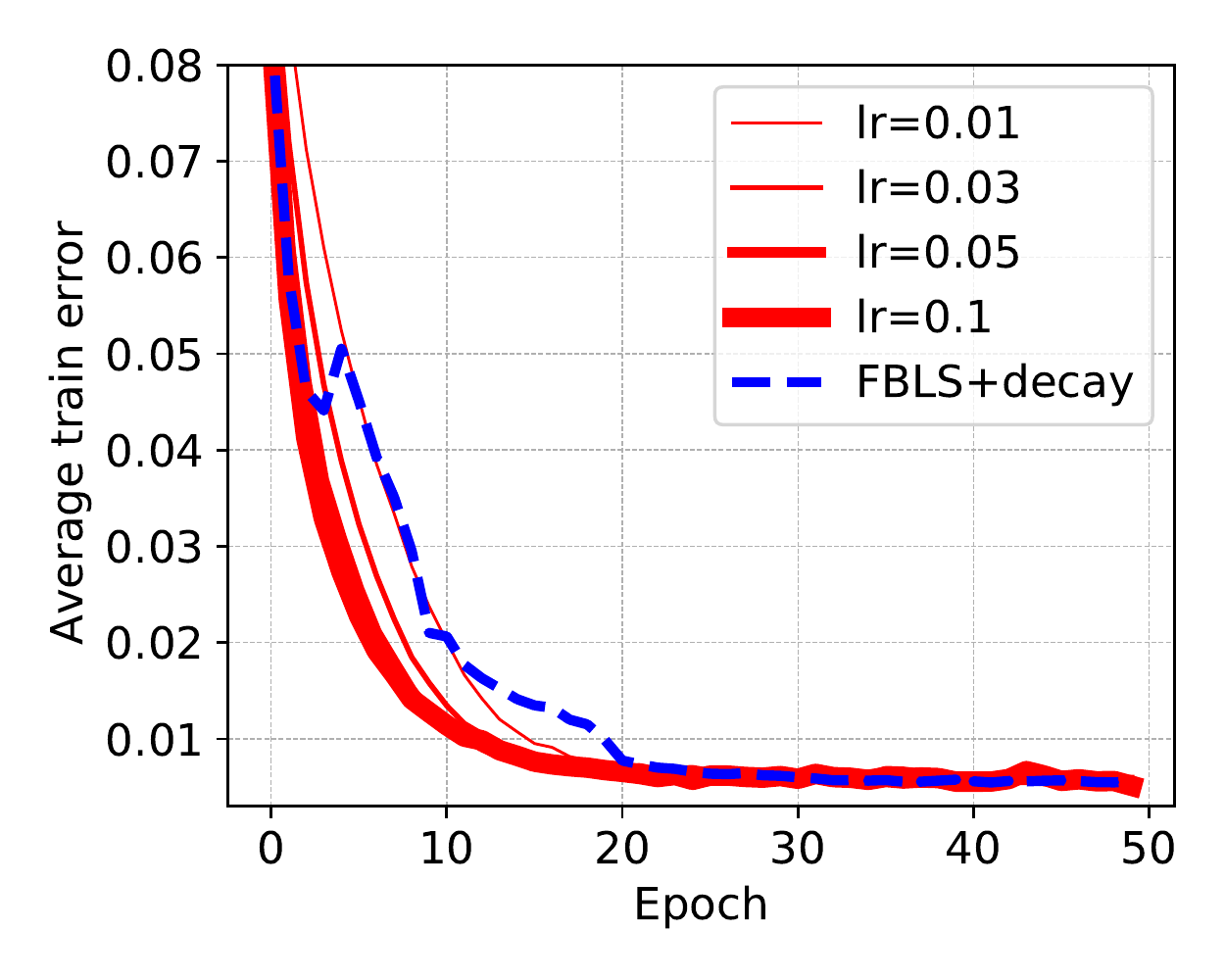}
            \caption{{\small \texttt{Train error}}}
    \end{subfigure}%
    \begin{subfigure}[b]{0.25\textwidth}
            \centering
            \includegraphics[width=\linewidth]{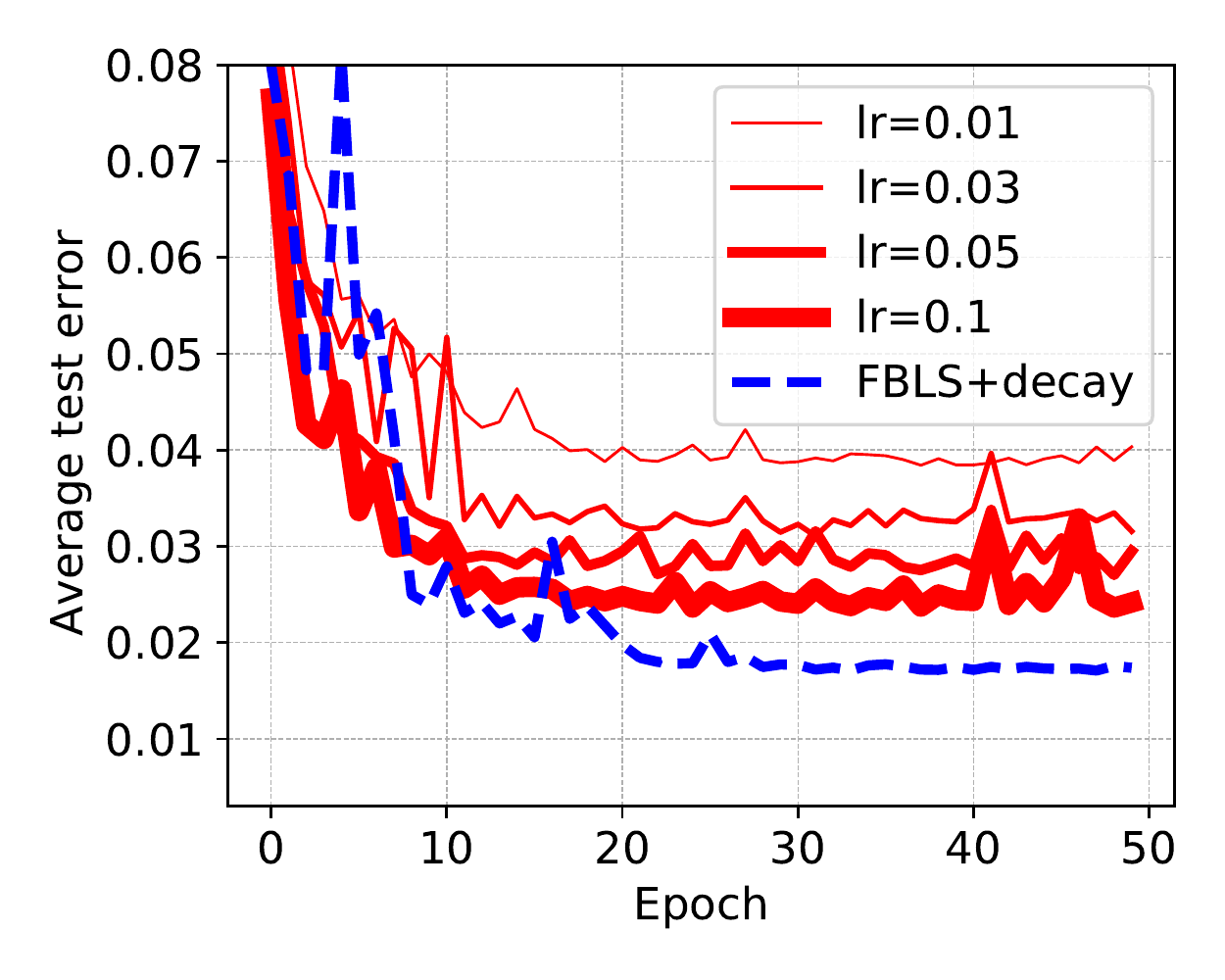}
            \caption{{\small \texttt{Test error}}}
    \end{subfigure}
    \caption{The learning curves for the different algorithms. Averaging here was on the different tasks. \textit{Top row:} MultiMNIST dataset. Data is averaged on $T = 2$ tasks and corresponding metrics. \textit{Bottom row:} CIFAR-10 dataset. Data is averaged on $T = 10$ tasks and corresponding metrics. These results demonstrate that FBLS often converges to points that are not worse than SGD for considered problems.}
\end{figure*}

For CIFAR-10, we use ten classes for creating ten synthetic "one vs rest" classification problems. The batch size has been set equal to 256. As an encoder, we have chosen scratch ResNet-18 and as decoder one fully connected layer.
We compare stochastic gradient and fast backtracking on these tasks (Figure \ref{fastbtmtl:fig:mcifar}). Eight step sizes for stochastic gradient descent have been chosen evenly from log-range from -3 to -1. As a loss, we chose cross-entropy, and as an error, we considered (1 - accuracy).

From Figure \ref{fastbtmtl:fig:mcifar} one can see that fast backtracking has the same final test loss as classical gradient descent. Also, fast backtracking is faster than classical gradient descent (see Table \ref{fastbtmtl:table:times}).

\subsection{Cityscapes}
As an encoder, we have chosen ResNet-50 pre-trained on Imagenet, and as decoders, we use PSPNet for disparity estimation, instance segmentation, and semantic segmentation. We compare stochastic gradient descent and fast backtracking. For  gradient we have chosen step sizes from $\{0.001, 0.01, 0.1\}$. 

\begin{table}[h!]
{\centering
\begin{tabular}{ccc}
\toprule
\textbf{Method } & \textbf{Average epoch time, $\cdot 10^3$ s} & \textbf{w/o} $\bm{\eta}$?  \\ 
\midrule
SGD   & 6.0  (100 \%) & $\bf{x}$\\ 
FBLS (Ours)  & 7.7 (128 \%) & \checkmark \\ 
\bottomrule
\end{tabular}
\caption{Time comparison for Cityscapes multi-task experiment. By "w/o~$\eta$?" we mean the ability to run an algorithm without selecting a learning rate $\eta$. The lower, the better. We took SGD as the baseline - 100 \%.}
\label{fastbtmtl:table:times_cityscapes}}
\end{table}

From the figures (see supplementary materials), we can see the competitive performance of fast backtracking against stochastic gradient descent. Also, we have a slight overhead in time by 28\%. Note that we need not select the initial learning rate in FBLS.

\subsection{BERT}
Transformers are widely used in tasks of natural language processing, as well as computer vision. In these models, the language model is the overwhelming part, which means that the time to compute the encoder forward/backward takes most of the optimization iteration. It would be especially interesting to study the FBLS algorithm for these models.

\begin{table}[h!]
{\centering
\begin{tabular}{ccc}
\toprule
\textbf{Method} & \textbf{Average epoch time, s} & \textbf{w/o} $\bm{\eta}$? \\ \midrule
SGD & 17.0 (100 \%) & $\bf{x}$ \\
BLS & 21.5 (126 \%) & \checkmark \\
MGDA-UB & 6.1 (36 \%) & $\bf{x}$\\
FBLS (Ours)  & 6.7 (39 \%) & \checkmark \\ 

\bottomrule
\end{tabular}
\caption{Time comparison for BERT on STS experiment. By "w/o~$\eta$?" we mean the ability to run an algorithm without selecting a learning rate $\eta$. The lower, the better. We took SGD as the baseline - 100 \%.}
\label{fastbtmtl:table:bert_times}}
\end{table}

For this reason, we have decided to measure epoch time for the Transformer-like model in a multi-task setting. We used BERT \cite{devlin2018bert} language model as the encoder on the STS benchmark \cite{cer2017semeval}. While the original benchmark was single-task oriented, we made an artificial multi-task setting with $T = 3$ same tasks. It means that we studied three independent decoders of the same architecture. Two fully connected layers were used as the decoder.

The results of Table \ref{fastbtmtl:table:bert_times} clearly show that the use of latent space approaches (FBLS and MGDA-UB) gives a significant reduction in the time per epoch compared to approaches based on entire encoder computation (SGD and BLS). The fastest method by time per epoch is MGDA-UB, but it inherits the disadvantage of SGD - the need to adjust the learning rate. As expected, FBLS is somewhat slower than MGDA-UB since some iterations of backtracking require several decoders forward/backward passes.

\section{Related work}
We recommend \cite{crawshaw2020multi, ruder2017overview} for surveys of modern multi-task learning. A classical approach to solve multiobjective optimization is a scalarization \cite{johannes1984scalarization}. Scalarization is a reduction of a multiobjective optimization problem to a single objective optimization problem. There are several techniques to provide it. The first one is weighting \cite{das1997closer}. In this case, one minimizes a weighted combination of objectives. The other approach is minimization on simplex constructed on minimums of every objective. These approaches included normal-boundary intersection method \cite{das1998normal}, normal constraints method \cite{messac2003normalized}. Also, there are several evolutionary strategies \cite{knowles2000approximating, deb2000fast}.

The other approach is based on the min-max technique. \cite{fliege2000steepest} uses min-max to standard gradient descent $L^2$-proximal task and gain method which output vector minimizing all direction evenly. The min-max technique was further adapted to Newton method \cite{fliege2009newton}, stochastic gradient \cite{fliege2011stochastic} and convergence rates were obtained \cite{fliege2019complexity}.

Initially, dual problem to aforementioned min-max approach convex hull norm minimization \cite{desideri2012multiple} became a new technique in multiobjective optimization. The method was modified in further works by adding the Gram-Schmidt process to gain solution \cite{desideri2012multiple}, modification the Gram-Schmidt process to gain fast inexact solution \cite{desideri2014multiple} and second-order method through normalization \cite{desideri2014multiple}. Also, it has been proved convergence in stochastic case \cite{mercier2018stochastic}. In \cite{sener2018multi} MGDA method was applied to latent space rather than parameter space. The replacement plays the role of the upper bound, and the solution in this space implies a solution in a shared parameter space. Adding task losses normalization \cite{katrutsa2020follow} in norm minimization creates a more robust solution in case of unbalanced tasks. The last method that can be categorized as gradient choosing method is \cite{yu2020gradient}. In this method, we change the initial solution cone to a proper "semi-dual" cone where every vector from this cone minimizes all losses and choose the average vector.

Besides gradient methods, adaptive weighting techniques were proposed to overcome the necessity of proper weighting search. \cite{kendall2018multi} used uncertainty of likelihood as weights. \cite{chen2018gradnorm} used loss term forcing gradient to update evenly. Also, \cite{liu2019end} used weighting based on rates of loss on the previous step and current step.

The separate direction in multi-task learning is architecture choice. It is considered that the choice of appropriate architecture allows introducing inductive bias. In \cite{long2017learning, yang2016deep} it is assumed what shared parameters have hidden tensor structure and examined tensor factorization and tensor normal distribution approaches. In \cite{misra2016cross, rosenbaum2017routing, ruder2019latent} controlling units were used. These units control the rate of sharing between intermediate encoder outputs. In \cite{lu2017fully, standley2019tasks} iterative neural architecture search was examined to build a multi-task model. In \cite{rebuffi2018efficient, meyerson2017beyond, liu2019end, maninis2019attentive} different attention mechanisms and adapters are used. These modules have fewer parameters than the main model, and they are used for solving multi-task and multi-domain problems. In \cite{zamir2018taskonomy} exhaustive investigation of relationships between computer vision tasks was worked out. 

\section{Conclusion}

In this work, we propose a novel optimization idea for multi-task learning. The idea is to use latent space to reduce the cost of line search. We examined this idea with a backtracking line search algorithm. We prove the convergence of the method theoretically. Also, we show the practical efficacy of the algorithm. The maximum efficiency of our method is achieved when the encoder is much larger than the decoder, which is valid on most NLP and CV models. 

The limitation of the Fast Backtracking Line Search is connected to SGD. We have to manually decay step size as it converges to the upper bound, but a more effective solution could exist.


\bibliography{bibliography}
\clearpage
\appendix


\section{Appendix}
\addtocounter{theorem}{-1}

\subsection{Theorem proof}
\theconv*

\begin{proof}

\begin{enumerate}
    \item Let $\bm{u}^t = \nabla_{\bm{\theta}^{t}} L^t$.  
    Let $\bm{\bar{\theta}}$ be an accumulation point of  $\{\bm{\theta}_k\}_{k=1}^{\infty}$. Then, there is a subsequence $\{\bm{\theta}_{k_j} \}_{j=1}^{\infty}$ converging to $\bm{\bar{\theta}}$. As $\hat{\bm{\gL}}$ continuous: $\hat{\bm{\gL}}(\bm{\theta}_{k_j}) \rightarrow \hat{\bm{\gL}}(\bm{\bar{\theta}})$. Consequently, $\eta\beta\|\bm{u}^t\|^2 \rightarrow 0$. There is an alternative:
    
    \begin{itemize}
        \item $\lim \sup \eta > 0$
        \item $\lim \eta = 0$
    \end{itemize}
    
In the first case, $\|\bm{u}^t\|^2 \rightarrow 0$, hence, $\frac{\partial \hat{\gL}^t}{\partial \bm{z}} = \bm{u}^t\frac{\partial \bm{\theta}^{t}}{\partial \bm{z}} \rightarrow 0$ and according to~\cite{sener2018multi} $\bm{\bar{\theta}}$~--- Pareto stationary point.

In the second case, assume that the accumulation point $\bm{\bar{\theta}}$ is not Pareto stationary point. Then, there is $\bar{\bm{d}}$~--- a direction which minimizes all functions $\hat{\gL}^t$. Since $\lim \eta = 0$, for every constant step size $\eta_n$, starting from some $j_0$ we can't satisfy Armijo condition for at least for one function:

\[
\forall \eta_n = \frac{1}{n} \ \exists j_0: \forall j \geq j_0 \ \exists t_n:
\]

\[
\hat{\gL}^{t_n}(\bm{z}_{k_j}-\eta_n\bm{d}_{k_j},  \bm{\theta}_{k_j}-\eta_n \bm{u}_j^{t_n}) \geq \hat{\gL}^{t_n}(\bm{z}_{k_j}) - \eta_n\beta\|\bm{u}_j^{t_n}\|^2
\] 

Since the sequence of indices $\{t_n\}$ is bounded by the number of tasks $T$, there is a subsequence of indices $\{t_{n_m}\}$ that converges to some index  $t_0$. 

Hence, starting from some $j_0$ we get that for $\hat{\gL}^{t_0}$ that $\forall j \geq j_0$:

\[
\hat{\gL}^{t_0}(\bm{z}_{k_j} - \eta_n \bm{d}_{k_j},  \bm{\theta}_{k_j}^{t_0} - \eta_n \bm{u}_{j}^{t_0}) \geq \hat{\gL}^{t_0}(\bm{z}_{k_j}) - \eta_n\beta\|\bm{u}_{j}^{t_0}\|^2
\]

As $\hat{\gL}^{t_0}$ is continuously differentiable and $\bm{\theta}_{k_j} \rightarrow \bm{\bar{\theta}}$ then  $\bm{d}_{k_j} \rightarrow \bm{d}$ and we get:

\[
    \hat{\gL}^{t_0}(\bm{z} - \eta_n\bm{d}, \bar{\bm{\theta}}^{t_0} - \eta_n\bm{u}^{t_0}) \geq \hat{\gL}^{t_0}(\bm{z}) - \eta_n\beta\|\bm{u}^{t_0}\|^2
\]

It is true for $\forall \eta_n = \frac{1}{n},$ where $n \in \mathbb{N}$. Thus, we get a contradiction with Armijo rule: since if $\bm{d}$~--- the minimizing direction, then $\exists \hat{\eta}: \ \forall t \in \{1, \ldots, T \}$:

\[
 \hat{\gL}^{t}(\bm{z}-\hat{\eta}\bm{d}, \bar{\bm{\theta}}^{t} - \hat{\eta} \bm{u}^{t}) \leq \hat{\gL}^{t}(\bm{z}) - \hat{\eta}\beta \|\bm{u}^{t}\|^2  - \hat{\eta}\beta \left(\frac{\partial \hat{\gL}^t}{\partial \bm{z}}\right)^T \bm{d}_z.
\] 

So $\bm{d} = \bm{0}$ and $\bm{\bar{\theta}}$~--- Pareto stationary point.
\item For this Armijo-rule, the proof can be obtained by the following modifications
As $\hat{\bm{\gL}}(\bm{\theta}_{k_j}) \rightarrow \hat{\bm{\gL}}(\bm{\bar{\theta}})$, then  $\eta\beta \left(\frac{\partial \hat{\gL}^t}{\partial \bm{z}}\right)^T \bm{d}~\rightarrow~0.$

In the first variant of alternative we have $\forall t \quad \left(\frac{\partial \hat{\gL}^t}{\partial \bm{z}}\right)^T \bm{d} = 0$. As $\|\frac{\partial \hat{\gL}^t}{\partial z}\|_{t=1}^T$ isn't a singular matrix, then $\bm{d} = \bm{0}$ and $\bm{\bar{\theta}}$~--- Pareto stationary point \cite{sener2018multi}.

In the second variant of alternative we can change all $\eta\beta \|\bm{u}^{t_0}\|^2$ to $\eta\beta \left(\frac{\partial \hat{\gL}^t}{\partial \bm{z}}\right)^T \bm{d}$ and the proof doesn't change.

\item For this Armijo-rule proof can be obtained by the following modifications.
As $\hat{\bm{\gL}}(\bm{\theta}_{k_j}) \rightarrow \hat{\bm{\gL}}(\bar{\bm{\theta}})$ we have $\eta\beta \left(\frac{\partial \hat{\gL}^t}{\partial \bm{z}}\right)^T \bm{d} + \eta\beta\|\bm{u}^t\|^2 \rightarrow 0 $. Both terms are non-negative, so we have the same alternative.

In the first variant of alternative we have that  $\bm{d} = \bm{0}, \bm{u}^t = \bm{0}$, so $\bm{\bar{\theta}}$~--- Pareto stationary point by \cite{sener2018multi}.

In the second variant of alternative we can change all  $\eta\beta \|\bm{u}^t\|^2$ to $\eta\beta \left(\frac{\partial L^t}{\partial \bm{z}}\right)^T \bm{d} + \eta\beta\|\bm{u}^t\|^2$ and the proof doesn't change.
\end{enumerate}
\end{proof}
\begin{remark}
By the Bolzano–Weierstrass theorem, under our conditions, there is at least one accumulation point. Thus, our algorithm's result will always be a Pareto stationary point.
\end{remark}
\subsection{Cityscapes}
This section presents the results on the Cityscapes dataset. Comparison of FBLS and SGD with fixed learning rate is presented on the Figure \ref{fast_bt:fig:cityscapes}.

\begin{figure*}[ht!]
    \centering 
\begin{subfigure}{0.31\textwidth}
  \includegraphics[width=\linewidth]{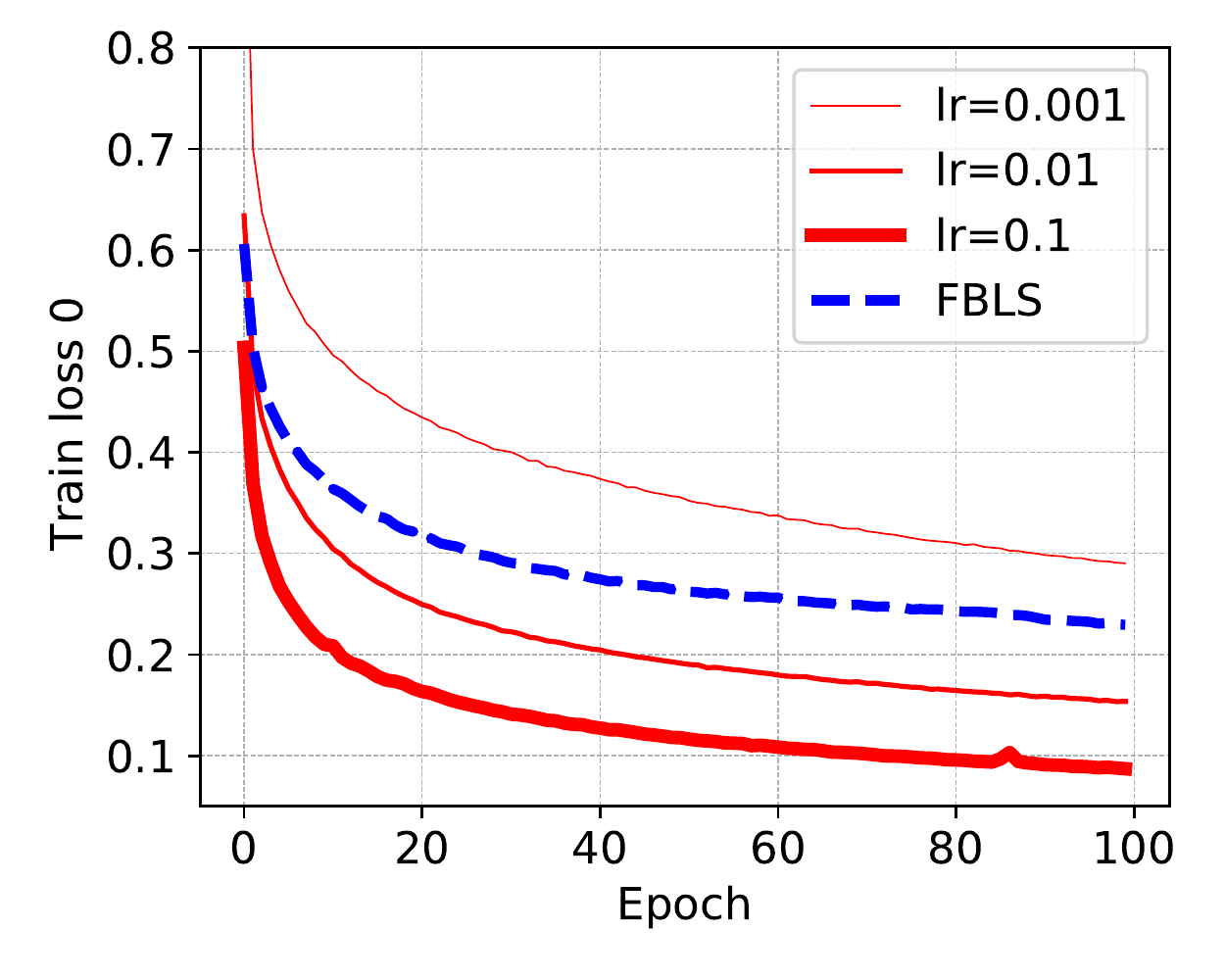}
\end{subfigure}\hfil 
\begin{subfigure}{0.31\textwidth}
  \includegraphics[width=\linewidth]{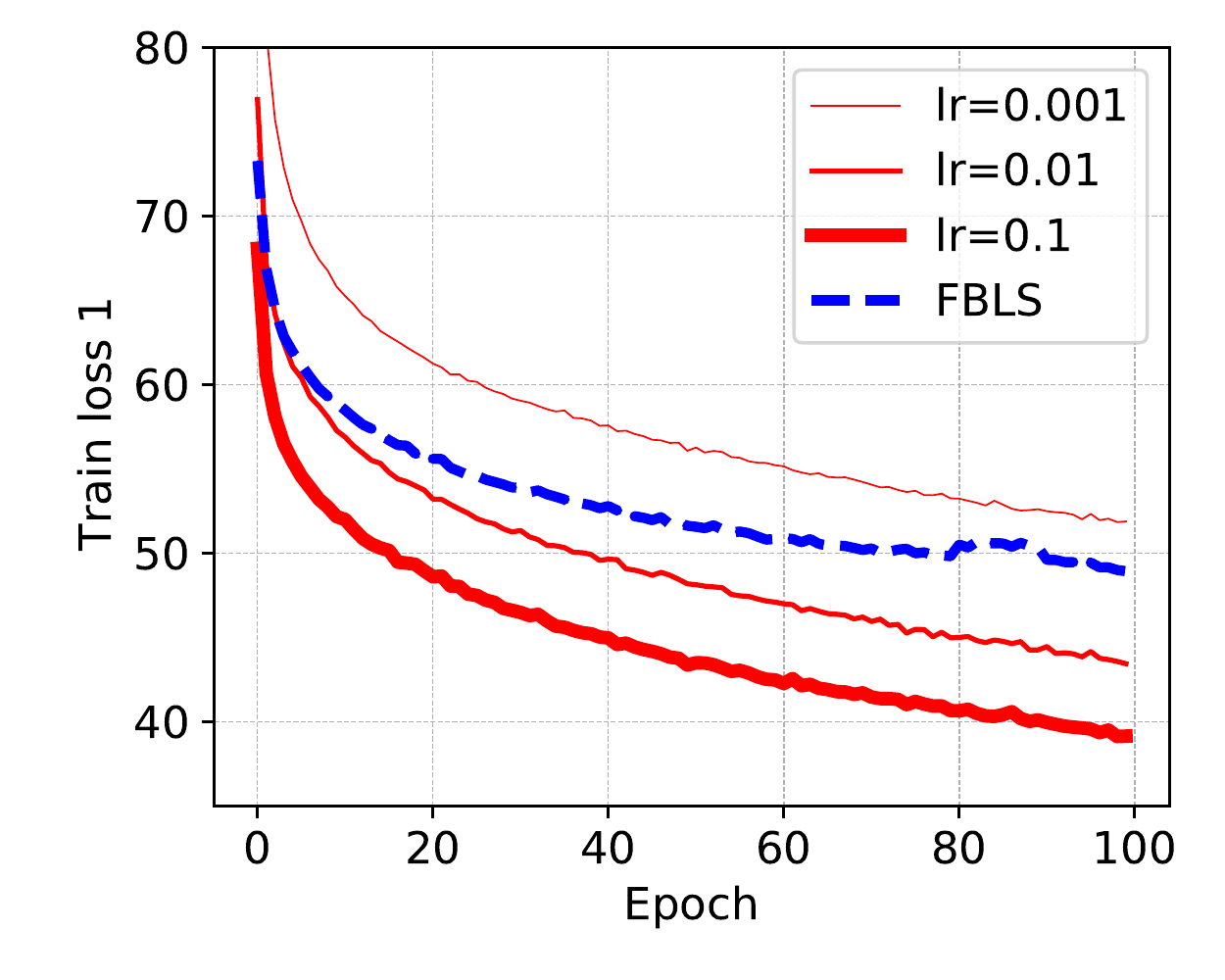}
\end{subfigure}\hfil 
\begin{subfigure}{0.31\textwidth}
  \includegraphics[width=\linewidth]{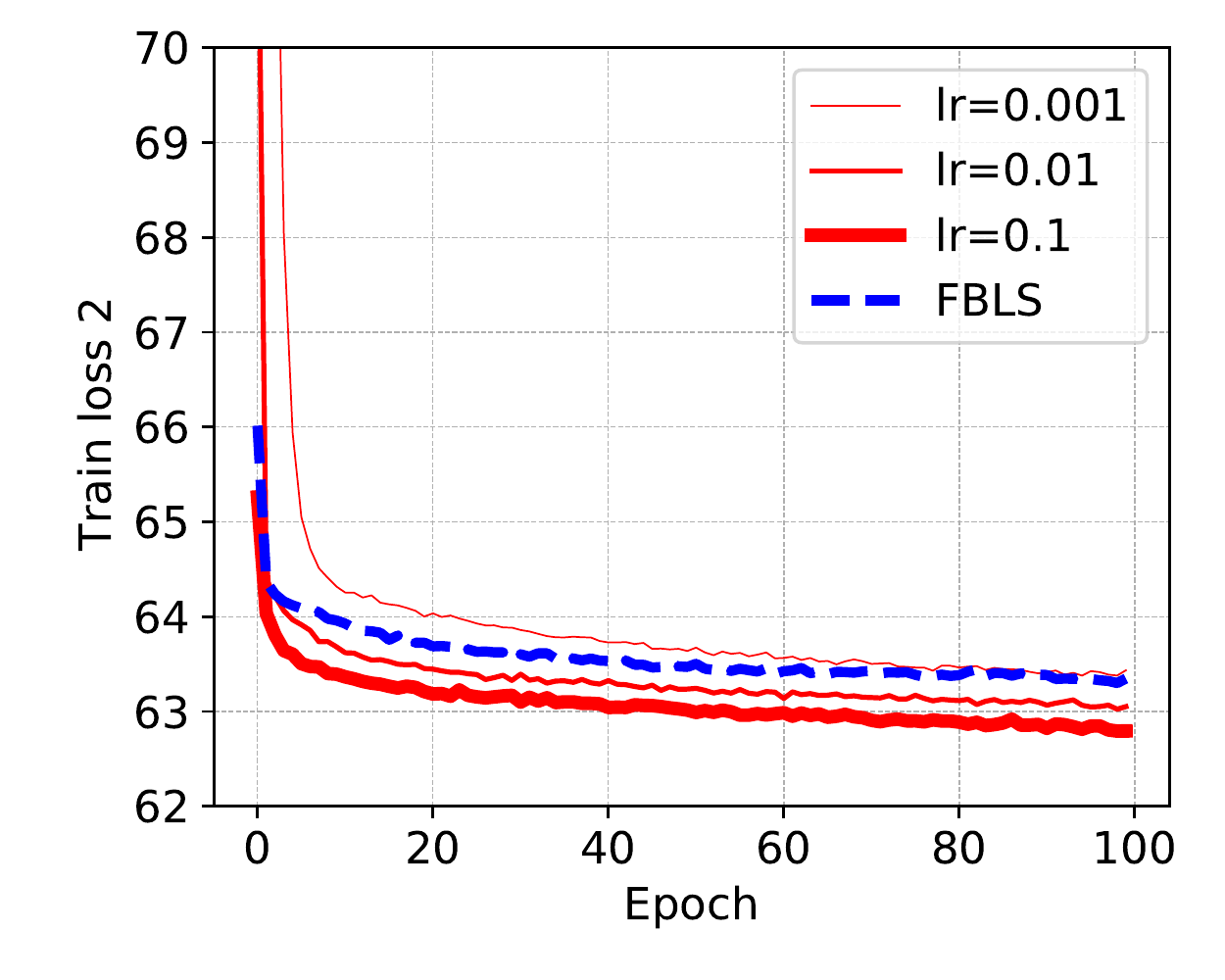}
\end{subfigure}
\begin{subfigure}{0.31\textwidth}
  \includegraphics[width=\linewidth]{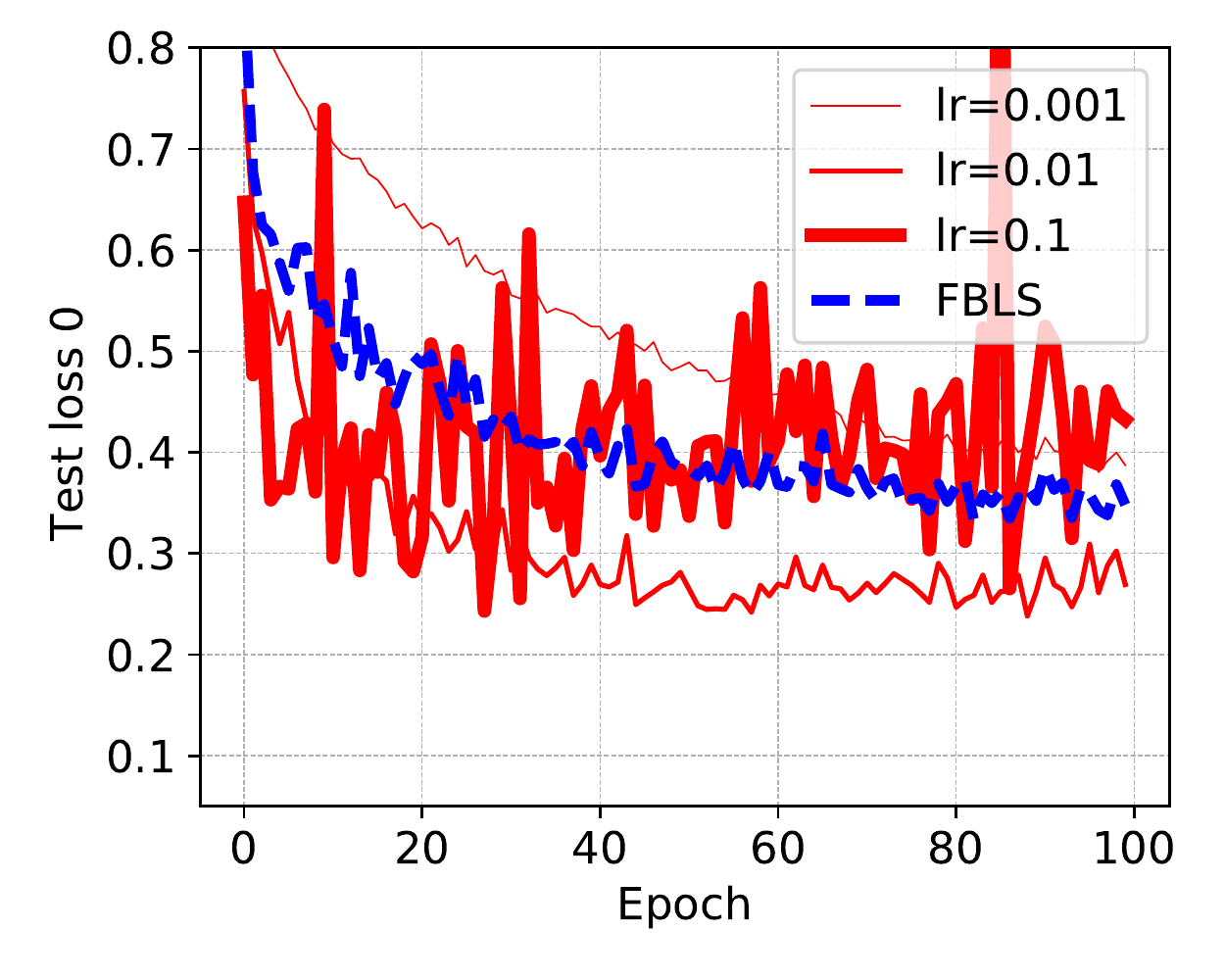}
  \caption{Visualization of train and test cross entropy loss for semantic segmentation task.}
  \label{fig:4}
\end{subfigure}\hfil 
\begin{subfigure}{0.31\textwidth}
  \includegraphics[width=\linewidth]{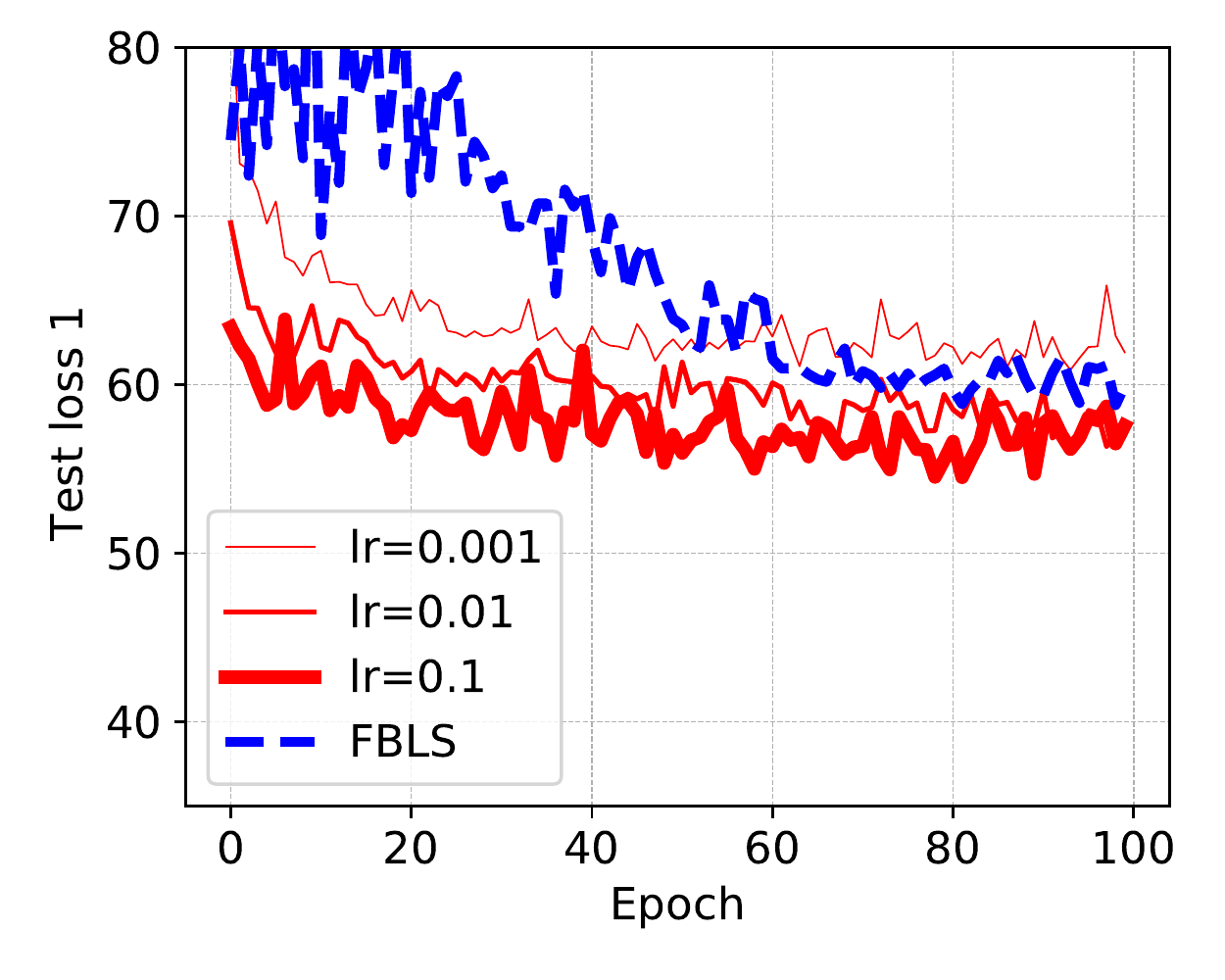}
  \caption{Visualization of train and test $l_1$ loss for instance segmentation task}
  \label{fig:5}
\end{subfigure}\hfil 
\begin{subfigure}{0.31\textwidth}
  \includegraphics[width=\linewidth]{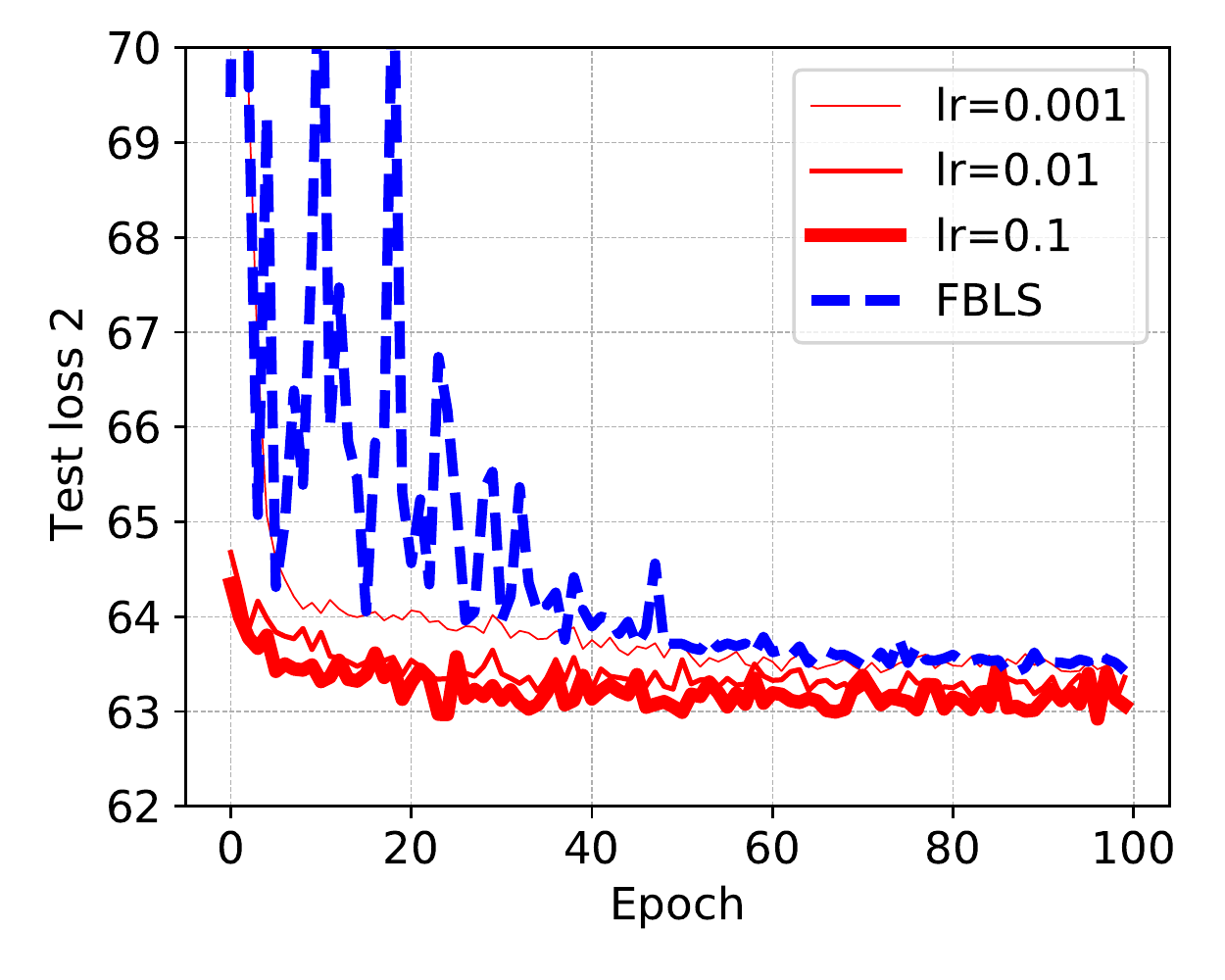}
  \caption{Visualization of train and test $l_1$ loss for disparity estimation task.}
  \label{fig:6}
\end{subfigure}
\caption{\textit{Top:} Train losses on cityscapes dataset. \textit{Bottom:} Test losses on cityscapes dataset.}
\label{fast_bt:fig:cityscapes}
\end{figure*}

\end{document}